\documentclass[10pt,twocolumn,letterpaper]{article}
\usepackage{authblk}
\usepackage[accsupp]{axessibility}
\usepackage[pagenumbers]{wacv} 
\usepackage{graphicx}
\usepackage{amsmath}
\usepackage{amssymb}
\usepackage{breqn}
\usepackage{booktabs}
\usepackage{amsthm}
\usepackage{algorithm}
\usepackage{soul}
\usepackage{algpseudocode}
\usepackage{multirow}
\usepackage{booktabs}
\usepackage{enumitem}
\usepackage{mathtools}
\newtheorem{theorem}{Theorem}
\newtheorem{prop}{Proposition}
\newtheorem{assumption}{A}
\newtheorem{lemma}{Lemma}
\usepackage[pagebackref,breaklinks,colorlinks]{hyperref}
\usepackage{adjustbox}

\usepackage[capitalize]{cleveref}
\crefname{section}{Sec.}{Secs.}
\Crefname{section}{Section}{Sections}
\Crefname{table}{Table}{Tables}
\crefname{table}{Tab.}{Tabs.}

\def\wacvPaperID{1253} 
\def\confName{WACV}
\def\confYear{2024}

\begin{document}

\title{Minimizing Layerwise Activation Norm Improves \\Generalization in   Federated Learning \vspace*{-0.7cm}}
\makeatletter
\renewcommand\AB@affilsepx{, \protect\Affilfont}
\makeatother
\author[1]{M. Yashwanth}
\author[2]{Gaurav Kumar Nayak}
\author[1]{Harsh Rangwani}
\author[3]{Arya Singh}
\author[1]{\\R. Venkatesh Babu}
\author[1]{Anirban Chakraborty\vspace*{-0.1in}}
\vspace{-0.25in}
{\affil[1]{Indian Institute of Science, Bangalore} 
\affil[2]{University of Central Florida}
\affil[3]{BITS Pilani \vspace*{-0.1in}}}
\vspace{-0.1in}
\affil[1] {\tt\small{\{yashwanthm,harshr,venky,anirban\}@iisc.ac.in}}
\affil[2] {\tt\small{gauravkumar.nayak@ucf.edu}\vspace{0.1in}}
\affil[3] {\tt\small{f20180762g@alumni.bits-pilani.ac.in}\vspace*{-0.1in}}
\maketitle
\begin{abstract}
Federated Learning (FL) is an emerging machine learning framework that enables multiple clients (coordinated by a server) to collaboratively train a global model by aggregating the locally trained models without sharing any client's training data. It has been observed in recent works that learning in a federated manner may lead the aggregated global model to converge to a `sharp minimum' thereby adversely affecting the generalizability of this FL-trained model. 
Therefore, in this work, we aim to improve the generalization performance of models trained in a federated setup by introducing a `flatness' constrained FL optimization problem. This flatness constraint is imposed on the top eigenvalue of the Hessian computed from the training loss. 
As each client trains a model on its local data, we further re-formulate this complex problem utilizing the client loss functions and propose a new computationally efficient regularization technique\footnote{\url{https://github.com/vcl-iisc/fedMAN.git}}, dubbed `\textit{MAN},' which \textbf{M}inimizes \textbf{A}ctivation's \textbf{N}orm of each layer on client-side models. We also theoretically show that minimizing the activation norm reduces the top eigenvalue of the layer-wise Hessian of the client's loss, which in turn decreases the overall Hessian's top eigenvalue, ensuring convergence to a flat minimum. We apply our proposed flatness-constrained optimization to the existing FL techniques and obtain significant improvements, thereby establishing new state-of-the-art.
\end{abstract}
\vspace{-0.2in}
\section{Introduction}
\label{sec:intro}
Federated Learning (FL), first introduced in~\cite{mcmahan2017communication}, is a distributed machine learning paradigm that involves multiple clients learning a shared model under the coordination of a central server. In FL, a client cannot send the training data to the server due to privacy concerns and communication overheads. Instead, clients share the parameters of models trained on their local data with the server, which aggregates these models with the objective of achieving better generalization on the overall data distribution across all clients. FL's privacy-preserving nature has made it increasingly popular in various domains, including smartphones~\cite{47586,ramaswamy2019federated}, Internet of Things~\cite{zhang2022federated,mills2019communication} and the healthcare industry~\cite{rieke2020future,xu2021federated,qayyum2022collaborative}.
\par

FedAvg introduced in~\cite{mcmahan2017communication} is a popular algorithm for federated training in which each client trains its local model for multiple epochs before sharing it with the server. However, FedAvg often exhibits convergence issues~\cite{karimireddy2020scaffold}, particularly when dealing with non-iid data. Numerous approaches have been proposed to overcome the limitations of FedAvg, particularly when dealing with non-iid distributions such as SCAFFOLD~\cite{karimireddy2020scaffold}, FedDyn~\cite{acar2021federated}, and FedDC~\cite{Gao_2022_CVPR}. These methods aim to address the problem from an optimization perspective by seeking a better minimum for the global model via minimization of the training objective. 
\par


  In centralized learning literature~\cite{yao2020pyhessian,keskar2017on, Sankar2021ADL,foret2021sharpnessaware}, a strong connection has been observed between generalization attained by the trained models and the sharpness of the loss surface where the model parameters converge at the end of training. Methods such as SAM~\cite{foret2021sharpnessaware} are proposed to attain flat minima for better generalization. 
 Recent studies on generalization under FL setup, such as FedSAM/ASAM~\cite{caldarola2022imp,qu2022generalized}, have shown that federated training often converges to a sharp minimum, which can negatively impact the generalization performance of the global model. This happens as each client participating in FL has a limited amount of data, and different clients can have different data distribution, eventually leading to overfitting of the models and sharp minimum.
 To solve this problem, methods such as FedSAM use SAM~\cite{foret2021sharpnessaware} optimizer on the client models to avoid the sharp minimum and get better generalized global models. 
 However, there are still issues with these methods-they require additional forward and backward passes as their update rules are based on gradient ascent followed by gradient descent. Methods such as FedDC~\cite{Gao_2022_CVPR} seem to perform better than  SAM-based FL optimizers such as FedSAM. Also FedSAM cannot be easily integrated atop the existing state-of-the-art optimizers such as FedDC to further enhance their performance without impacting the convergence guarantees. These methods also inherently assume that using SAM updates on the local models can provide a global model with flat minima. 
 \par
  
\par
To overcome the aforementioned limitations and motivated by the findings of centralized training, we present a novel Federated Learning (FL) optimization approach that integrates a flatness constraint as a regularizer on top of the FL objective. The top eigenvalue and trace of the Hessian of the loss are typical indicators of flatness~\cite{yao2020pyhessian,keskar2017on,foret2021sharpnessaware,Sankar2021ADL}. The lower values of the top eigenvalue/trace are desired for better performance. Ideally, our objective should be to achieve convergence towards a flat minimum while maintaining a low training loss. Thus, we directly aim to minimize the top eigenvalue of the global model along with the training loss. This facilitates convergence of the global model towards a flat minimum while simultaneously minimizing the training loss. By our choice of regularizer, we theoretically motivate that global model flatness can be achieved by local (client) model flatness, and develop a computationally efficient flatness metric that can minimize the top eigenvalue and be easily optimized by any of the state-of-the-art FL optimizers. As a result, our proposed approach can effectively improve the generalization performance of the trained global model. Our method also mitigates the need for multiple gradient updates as required by SAM-based FL methods, as the flatness constraints are integrated into the loss.


Further, we theoretically establish that the top eigenvalue of the Hessian of the client's loss function can be minimized by \textbf{M}inimizing the  \textbf{A}ctivation's \textbf{N}orm (MAN) of each layer in the client models (see Sec~\ref{theoritical_analysis}). Our empirical analysis in Sec.~\ref{emp_hess_analysis} shows that our proposed flatness constraint leads to a global model with lower top eigenvalue and the trace of the Hessian of loss, compared to the case without the constraint. Since our method achieves the flatness by minimizing the activation norm. Our proposed regularizer (MAN) can be elegantly combined with existing FL techniques such as FedAVG~\cite{mcmahan2017communication}, FedDC~\cite{Gao_2022_CVPR}, FedDyn~\cite{acar2021federated} and also SAM-based FL methods such as  FedSAM/ASAM~\cite{caldarola2022imp}, FedSpeed~\cite{sun2023fedspeed} etc. and further improve each of their performance to surpass the state-of-the-art on multiple datasets, as well as across the data distributions. For example, we improve the performance of FedDC by $3.2\%$ on CIFAR-100, and by $4.3\%$ on Tiny-ImageNet. For a detailed comparison, see Table~\ref{table_acc}. The key contributions of this work are:
\begin{itemize} 
    
    \item We propose a novel FL optimization problem that incorporates a flatness constraint as a regularizer to enhance generalization. Specifically, we minimize the top eigenvalue of the Hessian of the global model's training loss.
   
    \item 
    We present theoretical evidence that the top eigenvalue of the layerwise Hessian of client's loss can be minimized by minimizing its layer-wise activation norm. The computational complexity of our regularizer is very low compared to the existing sharpness-aware FL optimizers such as FedSAM/ASAM and FedSpeed.
    
    \item Unlike previous works that combine regularization schemes with only the FedAvg algorithm, our approach can be easily integrated atop any state-of-the-art methods, such as FedDC, FedDyn, FedSpeed, FedSAM/ASAM to further enhance their performance.
    
    \item We evaluate the effectiveness of our method on CIFAR-100 and Tiny-ImageNet datasets. Our method achieves significant improvements in FL-trained global model accuracy and communication cost reduction compared to competing methods. 
    
\end{itemize}

\section{Related Work}
FL was introduced in \cite{mcmahan2017communication} which is a generalized version of local SGD \cite{stich2018local}. This allows to increase the local gradient updates on the clients before sharing with server, which shows a significant reduction in communication costs in the iid setting but not in the non-iid setting. 
Methods such as MOON~\cite{li2021model} and FedProx~\cite{li2020federated} introduce a regularizer to maintain proximity to the global model while training with the local objective. The key idea is to prevent client models from drifting too far from the global model, thereby facilitating faster convergence or achieving the desired accuracy with minimal communication rounds. 
In~\cite{Mendieta_2022_CVPR}, a regularization strategy is proposed that minimizes the Lipschiltz constant, which involves high compute.  
SCAFFOLD \cite{karimireddy2020scaffold} introduced the gradient correction to minimize client drift. Later FedDyn \cite{acar2021federated} improved upon this by introducing the dynamic regularization term. This was improved in FedDC~\cite{Gao_2022_CVPR} by local drift correction. In \cite{zhao2018federated}, the data-sharing strategy was proposed to improve the performance, which might violate the privacy in FL. Distillation strategies are proposed in FedNTD~\cite{lee2021preservation}, FedSSD~\cite{he2022learning} and FedCAD~\cite{he2022class} to address the problems associated with non-iid settings. However, these methods are computationally intensive, and FedCAD also requires access to external data. In \cite{mishchenko2019distributed}, models were compressed to reduce communication cost. One-shot methods, e.g.,~\cite{zhou2020distilled} were also proposed to reduce communication cost but these incur heavy computation on client.

Recent works such as~\cite{caldarola2022imp,qu2022generalized} showed that Federated training converges to a sharp minimum. To avoid it, these methods seek flat minimum for client models using SAM~\cite{foret2021sharpnessaware}, ASAM~\cite{kwon2021asam} and SWA~\cite{izmailov2018averaging} in the FL setting. Sharpness Aware minimization (SAM) is an optimization technique that improves the generalization of models by converging to flat minima across various tasks~\cite{foret2021sharpnessaware, rangwani2022closer, rangwani2022escaping}. Recently, FedSpeed~\cite{sun2023fedspeed} has been proposed, which can be viewed as a combination of SAM with FedDyn~\cite{acar2021federated}. However, there exist several limitations of SAM-based FL methods. The lack of theoretical results for the convergence of ASAM/SWA in FL is one of them. Moreover, SAM-based FL methods necessitate the tuning of multiple hyper-parameters. In our framework, we address these limitations by introducing a novel problem formulation that helps theoretically connect the flat minimum of each client model to that of the global model as a part of our optimization framework. 
We will now explain our method in detail in the next section.
\section{Method}
We first describe the FL optimization problem involving multiple clients and a single server in Sec.~\ref{psetup}. Next, we describe our problem formulation incorporating flatness constraints along with the minimization of overall training loss (cross-entropy). Subsequently, we re-write the objective in terms of the each client's training loss and top eigenvalue of the Hessian of the client's loss in Sec.~\ref{prop_method_flat_constraint}. We offer theoretical insights in Sec.~\ref{overall_layer} and Sec.~\ref{theoritical_analysis}  to establish the relationship between activation norms and the layer-wise Hessian eigenvalues, which explains the effectiveness of our method. Based on our theoretical insights, we minimize activation norm along with CE loss to attain flatness in Sec.~\ref{man_regularizer}. In Sec.~\ref{emp_hess_analysis}, we present empirical evidence that validates our method by showing that it indeed minimizes the top eigenvalue/trace of the Hessian of the global model's training loss.
\subsection{Problem Setup}
\label{psetup}
In the FL optimization problem, we consider a scenario in which a single server interacts with $n$ clients or edge devices. It is further assumed that each client $k$ possesses $m_{k}$ training samples drawn iid from the data distribution $\mathcal{D}_{k}(x,y)$. The data distributions $\{\mathcal{D}_{k}(x,y)\}_{k=1}^n$ across the clients can be either iid or non-iid. 

\begin{equation}
\underset{\mathbf{w}}{\arg\min} \hspace{0.08in}l(\mathbf{w})
\label{fl_opt}
\end{equation}
\vspace{-0.1in}
where 
\begin{equation}
l(\mathbf{w}) =  \ \frac{1}{n} \sum_{k = 1}^{n}{L_{k}(\mathbf{w})}
\label{L_w_eqn}
\end{equation}
where $ L_{k}(\mathbf{w}) $ is the client specific objective function and $\mathbf{w}$ denotes model parameters. $L_{k}(\mathbf{w})$ is defined as follows.
\begin{equation}
L_k(\mathbf{w}) = \underset{x,y \in D_k}{\mathbb{E}}[l_{k}(\mathbf{w};(x,y))]   
\end{equation} 
Here, $l_k$ is cross-entropy loss. The expectation is approximated by averaging over training samples drawn from data distribution ($\mathcal{D}_{k}$) of the client $k$. Optimizing $L_k(\mathbf{w})$ directly (i.e., FedAvg) leads to solutions in a sharp minimum. The local models over-fit to client's data distribution and do not generalize well. To mitigate this, we introduce a novel regularizer that induces flatness to global model, which we elucidate in the next section.
\subsection{Proposed Method}
Minimizing the objective function in Eq.~\ref{fl_opt} directly, may lead to global model converging to the sharp minimum which can affect the generalization performance~\cite{caldarola2022imp,qu2022generalized}. To avoid this, we explicitly introduce the flatness constraints by minimizing the top eigenvalue of the Hessian of the training loss of the global model. Our approach is motivated by the works of centralized setup, where it has been shown that top eigenvalue of Hessian of loss is a typical indicator of flatness~\cite{yao2020pyhessian,keskar2017on,foret2021sharpnessaware, Sankar2021ADL}.     \label{prop_method}
\subsubsection{Problem Formulation with Flatness Constraint}
\label{prop_method_flat_constraint}
We now formulate the problem of FL as a constrained optimization problem with the flatness constraint.
We rewrite the Eq.~\ref{fl_opt} with the constraint as below
\begin{equation*}
\underset{\mathbf{w}}{\arg\min} \hspace{0.08in}l(\mathbf{w})
\end{equation*}
\vspace{-0.05in}
such that
\begin{equation}
 \lambda_{max}(\textbf{H}(l(\mathbf{w})))< \tau
 \label{sm_constraint}
\end{equation}
where $\lambda_{max}(\textbf{H}(l(\mathbf{w})))$
denotes the top eigenvalue of the Hessian of the loss $l(\mathbf{w})$.
\footnote{$\textbf{H}(l(\mathbf{w}))$ denotes the Hessian of the loss $l(\mathbf{w})$.}
We now translate the constraint in above problem using the penalized loss in Eq.~\ref{lagrang_opt}, with hyper-parameter $\zeta > 0$ balancing the loss and flatness.
\vspace{-0.02in}
\begin{equation}
\underset{\mathbf{w}}{\arg\min} \hspace{0.08in} \hat{l}(\mathbf{w}) =  l(\mathbf{w}) + \zeta \lambda_{max}(\textbf{H}(l(\mathbf{w})))
\label{lagrang_opt}
\end{equation}
\vspace{-0.01in}
Its hard to optimize the term $\lambda_{max}(\textbf{H}(l(\mathbf{w})))$ for two reasons. Firstly, it needs access to loss functions of all the clients which in FL is local to each client, hence not accessible. Secondly, it needs to compute the eigenvalue of the Hessian of the loss, which is computationally complex. 
To address the first issue, we simplify optimization in Eq.~\ref{lagrang_opt} by upper bounding the $\lambda_{max}(\textbf{H}(l(\mathbf{w})))$ as described below. 

As Hessian is Jacobian of the gradient, we get 
\begin{equation}
\textbf{H}(l(\mathbf{w})) = {\frac{1}{n}} {\sum_{k} \textbf{H}(L_k(\mathbf{w}))}
\label{distr_hess}
\end{equation}
We use the identity\footnote{$\lambda_{max}(\textbf{S}_1 + \textbf{S}_2) \leq  \lambda_{max}(\textbf{S}_1) + \lambda_{max}(\textbf{S}_2)$ where $\mathbf{S}_1$ and $\mathbf{S}_2$ are symmetric matrices.} in the Eq.~\ref{distr_hess} to get~\ref{eig_ineq_hess} 
\begin{equation}
\lambda_{max}(\textbf{H}(l(\mathbf{w}))) \leq {\frac{1}{n}} {\sum_{k} \lambda_{max}(\textbf{H}(L_k(\mathbf{w})))}
\label{eig_ineq_hess}
\end{equation}
Using inequality~\ref{eig_ineq_hess} and Eq.~\ref{L_w_eqn} we upper bound the loss $\hat{l}(\mathbf{w})$ in Eq~\ref{lagrang_opt} by the following. 
\begin{equation}
\tilde{l}(\mathbf{w}) = {\frac{1}{n}}\sum_{k}(L_k(\mathbf{w}) + \zeta \lambda_{max}(\textbf{H}(L_k(\mathbf{w}))))
\label{hess_loss_ub}
\end{equation}
i.e.,
$
\hat{l}(\mathbf{w}) \leq \tilde{l}(\mathbf{w})
$.
We now minimize the $\tilde{l}(\mathbf{w})$ as below 
\begin{equation}
\underset{\mathbf{w}}{\arg\min} \hspace{0.08in} \tilde{l}(\mathbf{w}) = {\frac{1}{n}}\sum_{k}L_k(\mathbf{w}) + \zeta \lambda_{max}(\textbf{H}(L_k(\mathbf{w})))
\label{clnt_obj_hess}
\end{equation}
Minimizing upper bound $\tilde{l}(\mathbf{w})$ will minimize the $\hat{l}(\mathbf{w})$, which is our true goal.
We therefore modified the problem in Eq.~\ref{fl_opt} by including the flatness constraints in Eq.~\ref{clnt_obj_hess}. Instead of just minimizing the $L_k(\mathbf{w})$ each client now minimizes the $L_k(\mathbf{w}) + \zeta \lambda_{max}(\textbf{H}(L_k(\mathbf{w})))$. The $\zeta$ trades off between the flatness and the training loss. We now address the second issue involving  the top eigenvalue of Hessian which is computationally complex by developing a low-cost metric to minimize the $\lambda_{max}(\textbf{H}(L_k(\mathbf{w})))$.
\subsubsection{Bounding the overall Hessian Eigenvalues with the Layer-Wise Hessian Eigenvalues}
\label{overall_layer}
According to the findings of~\cite{Sankar2021ADL}, during the training process the evolution of the top eigenvalue of the overall Hessian is similar to the top eigenvalue of layer-wise Hessian. By layer-wise Hessian, we mean the Hessian of loss is computed with respect to parameters of each layer. We explain the above behavior through the following result. 

\begin{theorem}
If $\mathbf{H}_{ll} \in R^{d_l}$  denotes the layer $l$ Hessian and $\mathbf{H} \in R^{d}$ denotes the over all Hessian and $\sum_{l=1}^{L}{d_l} = d$, where $L$ is the total number of layers. If the Hessian entries are bounded above we then have the following result.
$\lambda(\mathbf{H}) \in \cup_{l=1}^{L} [ \lambda_{min}(H_{ll}) - \mathcal{O}(max(d_l,{d-d_l})), \lambda_{max}(H_{ll}) + \mathcal{O}(max(d_l,{d-d_l}))] $
\label{hess_theorem0}
\end{theorem}

The detailed proof is given in the Sec.\textcolor{red}{2} of supplementary. The theorem says that all the eigenvalues of the overall Hessian $\lambda(\mathbf{H})$ can be upper bounded by the layerwise Hessian's top eigenvalue along with a constant dependent on dimensions. In light of the empirical observations in~\cite{Sankar2021ADL} and above theoretical findings, we opt to minimize the layer-wise Hessian's top eigenvalue instead of minimizing the overall Hessian's top eigenvalue.

\begin{figure*}[t]
    \centering
    \subfloat[FedAvg]{
    \includegraphics[scale=0.43]{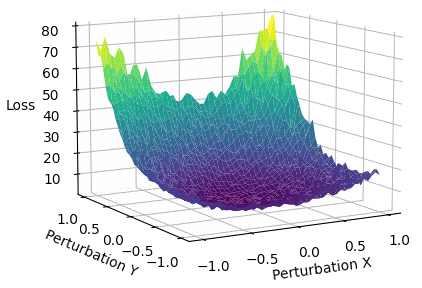}
     \label{fedavg_hess}
    }
    \subfloat[MAN]{
    \includegraphics[scale=0.45]{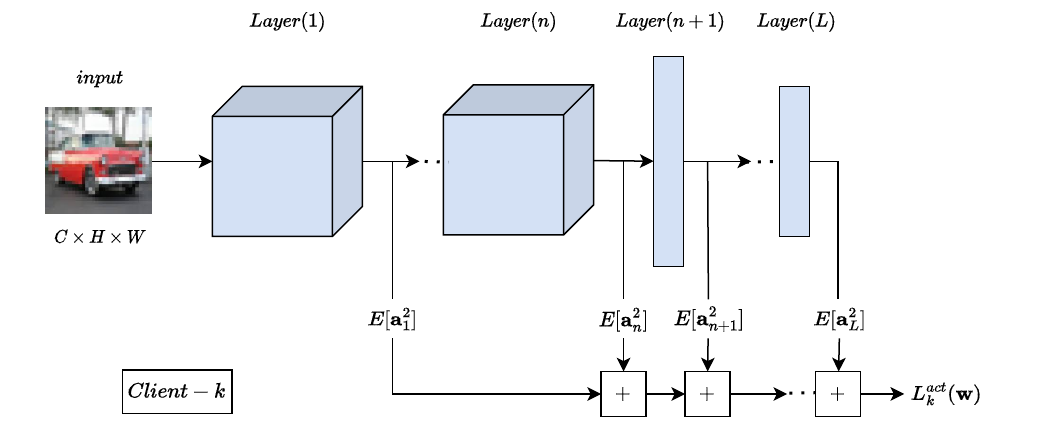}
     \label{man_comp}
    }
    \subfloat[FedAvg+MAN]{
    \hspace{-0.8cm}
    \includegraphics[scale=0.43]{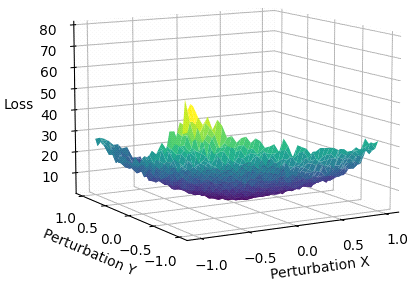}
    \label{fedavgreg_hess}
  }
  
  \caption{
   We plot the loss surface of the global model trained on CIFAR-100 using FedAvg in~\ref{fedavg_hess}. In Fig~\ref{man_comp} we show MAN regularizer. We combine FedAvg with MAN (FedAvg+MAN) to obtain the flat loss surface in Fig~\ref{fedavgreg_hess} which has better generalization.
  }
  \label{main_fig}
\end{figure*}

\subsubsection{Bounding the Layer-Wise Hessian Eigenvalues with the Activation Norms}
\label{theoritical_analysis}
Minimizing the layer-wise Hessian top eigenvalue is still a challenging task. To address this issue, we further simplify our method based on our results, which indicate that the top eigenvalue of the layer-wise Hessian is upper-bounded by the activation norm of the inputs to that layer. Therefore, by minimizing the activation norm, we can minimize the top eigenvalue of the layer-wise Hessian.

In our analysis, we consider the $C$ class classification problem. The training set $S = \{(\mathbf{x}^i,\mathbf{y}^i)\}_{i=1}^{N}$ is considered, where each $\mathbf{x}^i \in \mathbb{R}^D$ or $\mathbf{x}^i \in \mathbb{R}^{C\times H \times W}$and $\mathbf{y}^i \in \{0,1\}^C$ is drawn iid from the distribution $\mathcal{D}$. We then consider an $L$-layer neural network, where the network outputs the logits $\mathbf{z}^i$. The logits are obtained by a series of fully connected (FC) / convolutional (CONV) layers, followed by a non-linearity, represented concisely by Eq.~\ref{forward_affine} for a FC layer with parameters ($\{\mathbf{W}_l,\mathbf{b}_{l}\}$) and Eq.~\ref{forward_conv} for a CONV layer with parameters ($\{\mathbf{W}_l,\mathbf{b}_{l}\}$).
\begin{equation}
\mathbf{z}_l^i = \text{FC}(\mathbf{a}_{l-1}^i;\{\mathbf{W}_l,\mathbf{b}_{l}\})
\label{forward_affine}
\end{equation}
\begin{equation}
\mathbf{z}_l^i = \text{CONV}(\mathbf{a}_{l-1}^i;\{\mathbf{W}_l,\mathbf{b}_{l}\})
\label{forward_conv}
\end{equation}
\begin{equation}
\mathbf{a}_l^i = \sigma(\mathbf{z}_l^i) 
\end{equation}
where $\sigma(.)$ denotes non-linearity $\mathbf{a}_{0} = \mathbf{x}^i$, the logits $\mathbf{z}^i = \mathbf{a}_L^i$. 
We denote the output of softmax on logits $\mathbf{z}^i$ as 
\begin{equation}
\hat{\mathbf{y}}=  exp(\mathbf{z}^i)/\sum_{m=1}^{C}{exp(\mathbf{z}^i[m])})
\label{softmax_eqn}
\end{equation}
where $\hat{\mathbf{y}} \in [0,1]^C$
Finally, we use the cross-entropy loss 
\begin{equation}
\mathcal{L}(\mathbf{y}^i,\hat{\mathbf{y}}^i) = \sum_{c=1}^{C}{-\mathbf{y}^i[c] log(\hat{\mathbf{y}}^i[c])}.
\label{softmax_eq}
\end{equation}
where $\mathbf{x}^i$, $\mathbf{y}^i$ denotes the $i^{th}$ input sample and label respectively. 
We use the notation $\mathcal{L}^i$ for $\mathcal{L}(\mathbf{y}^i,\hat{\mathbf{y}}^i)$. The overall loss computed on Batch size of B is denoted by 
$\mathcal{L}={\frac{1}{B}}\sum_{i=1}^{B}{\mathcal{L}^i}$.  
Finally, we denote ${\theta} \in \mathbb{R}^d$ as the collection of all the model parameters into a vector of dimension $d$ that denotes the total number of parameters in the network. We explicitly denote Hessian of layer $l$ as {$\mathbf{H}_{\textbf{W}_{l}}({\mathcal{L}}$)}
We now state the two of our results that relate the layer-wise top eigenvalues to the activation norm of each layer. Consider a FC layer as in Eq.~\ref{forward_affine} with $\mathbf{a}_{l-1}^i \in \mathbb{R}^{d_{l-1}}$ and weights $\mathbf{W}_l \in \mathbb{R}^{d_{l-1}\times d_l}$. We then have the following result.    
\begin{theorem}
If ${\lVert \theta \rVert}_2 \leq \tilde{B}$, then the top eigenvalue of layer-wise Hessian for the loss $\mathcal{L}$ w.r.t to $ \textbf{W}_l$ denoted by $\lambda_{max}(\mathbf{H}_{\textbf{W}_{l}}{(\mathcal{L})}) $ for l = $2$ to $L$,  computed over the batch of samples for a $L$ layered fully connected neural network for multi-class classification is given by $\lambda_{max}({\mathbf{H}_{\textbf{W}_{l}}({\mathcal{L}}))   \leq  \alpha_l \sum_{i \in B} {\lVert \mathbf{a}_{l-1}^{i} \rVert}_2^2}$ where $\alpha_l>0$.
\label{hess_theorem1}
\end{theorem}
Where ${\lVert . \rVert}_{2}$ denotes the Euclidean norm. We now present a similar result for the CONV layer with input feature map of dimension $\mathbf{a}_{l-1}^i \in \mathbb{R}^{C_{l-1} \times H_{l-1} \times W_{l-1}}$, the output feature map $\mathbf{z}_l^i \in \mathbb{R}^{m \times H_l \times W_l}$ and convolutional kernel $\mathbf{W}_l \in \mathbb{R}^{m \times C_{l-1}\times K_1 \times K_2}$, then the below Theorem holds.
\begin{theorem} 
If ${\lVert \theta \rVert}_2 \leq \tilde{B}$, then the top eigenvalue of layer-wise Hessians for the loss (w.r.t to $ \textbf{W}_l$ for l = $2$ to $L$ ) computed over the Batch of samples for a $L$ layered convolutional neural network for multi-class classification is given by $\lambda_{max}({\mathbf{H}_{\textbf{W}_{l}}{(\mathcal{L})}) \leq  \alpha_l\sum_{i \in B}  {\lVert \mathbf{a}_{l-1}^{i} \rVert}_F^2}$ where $\alpha_l^i>0$.
\label{hess_theorem2}
\end{theorem}
where ${\lVert . \rVert}_{F}$ is the Frobenious norm\footnote{${\lVert \mathbf{a}_{l-1}^i \rVert}_{F}^2 \coloneq \sum_{j_1 = 0} ^ {C_{l-1}-1} \sum_{j_2 = 0} ^ {H_{l-1}-1} \sum_{j_3 = 0} ^{W_{l-1}-1}({\mathbf{a}_{l-1}^i}[j_1,j_2,j_3])^2$}. The result of Theorem~\ref{hess_theorem1} and  Theorem~\ref{hess_theorem2} states that  for a fully connected layer or a convolutional layer the top eigenvalue of Hessian of layer $l$ can be minimized by minimizing the activation norm of its input $ \mathbf{a}_{l-1}^{i}$. Detailed proofs of Theorems \ref{hess_theorem1} \& \ref{hess_theorem2} are provided in the Sec.\textcolor{red}{2} of the supplementary.
\vspace{-0.05in}
\subsubsection{Minimizing the Activation Norm}
\label{man_regularizer}

Building on the insights from the aforementioned findings, we address our objective by minimizing the layer-wise activation norms. This, in turn, minimizes the top eigenvalue of the layer-wise Hessian, thereby achieving a reduction in the top eigenvalue of the overall Hessian (i.e. increase flatness). We now present our proposed method, which involves each client $k$ minimizing $f_{k}(\mathbf{w})$ as defined below (Eq.~\ref{eq1}).
\begin{equation}
f_{k}(\mathbf{w}) \triangleq L_{k}(\mathbf{w}) + \zeta L_{k}^{act}(\mathbf{w})
\label{eq1}
\end{equation}
$\zeta$ is the balancing parameter between the task-specific loss and the regularization loss. We use $L_{k}^{act}$ as the substitute for $\lambda_{max}(\textbf{H}(L_k(\mathbf{w})))$ in the Eq.~\ref{clnt_obj_hess}.
Our proposed regularizer $L_{k}^{act}(\mathbf{w})$ computes the second moment of activation's for a layer ($l$) (denoted by $\mathbf{a}_l$) after non-linearity and then further sum across the $L$ layers. The computation of the activation norm and its impact is shown in Figure~\ref{main_fig}. Mathematically we describe our regularization term as below 
\begin{equation}
 L_{k}^{act}(\mathbf{w}) \triangleq  \sum_{l=1}^{L} \mathop{\mathbb{E}}[{\lVert \mathbf{a}_l \rVert}^2]
\label{man_eq}
\end{equation}
 $\mathop{\mathbb{E}}$ denotes the expectation operation approximated by averaging over mini-batch. $l$ denotes the layer index which can be a convolutional layer or a fully connected layer. If $l$ is the CONV layer, $\mathbf{a}_l^i \in \mathbb{R}^{C_l \times H_l\times W_l}$ is activation map of $i^{th}$ sample.  we define $\mathop{\mathbb{E}}[{\lVert \mathbf{a}_l \rVert}^2]$ in Eq.~\ref{conv_mean_eq}.
\begin{equation}
\mathop{\mathbb{E}}[{\lVert \mathbf{a}_l \rVert}^2] \coloneqq  {1 \over BC_lH_lW_l} \sum_{i = 0} ^ {B-1} \sum_{j_1 = 0} ^ {C_l-1} \sum_{j_2 = 0} ^ {H_l-1} \sum_{j_3 = 0} ^{W_l-1}({\mathbf{a}_l^i}[j_1,j_2,j_3])^2
\label{conv_mean_eq}
\end{equation}
In the above equation $B$ is the batch size, $C_l$ is the feature map, $H_l$ and $W_l$ denotes the height and width of the activations at layer $l$. This can be compactly written as 
$\mathop{\mathbb{E}} [{\lVert \mathbf{a}_l \rVert}^2] =  {1 \over BH_lW_l C_l} \sum_{i = 0} ^ {B-1} {\lVert \mathbf{a}_l^i \rVert}_{F}^2$.
If $l$ is the FC layer, $\mathbf{a}_l^i \in \mathbb{R}^{d_l}$, $\mathop{\mathbb{E}} [{\lVert \mathbf{a}_l \rVert}^2] $ is defined below
\begin{equation}
\mathop{\mathbb{E}} [{\lVert \mathbf{a}_l \rVert}^2] \coloneqq  {1 \over Bd_l} \sum_{i = 0} ^ {B-1} \sum_{j_1 = 0} ^ {d_l-1} (\mathbf{a}_l^i[j_1])^2
\label{fc_mean_eq}
\end{equation}
$B$ is the batch size and and $d_l$ is feature dimension of the FC layer. Eq.~\ref{fc_mean_eq} can also be written as 
$\mathop{\mathbb{E}} [{\lVert \mathbf{a}_l \rVert}^2]  =  {1 \over B} \sum_{i = 0} ^ {B-1} 
{1 \over d_l} {\lVert \mathbf{a}_l^i\rVert}_{2}^2$, where ${\lVert . \rVert}_2$ is the euclidean norm.   

\subsubsection{Convergence Analysis}
We now analyze the convergence of MAN regularizer when its used by the clients as in Eq.~\ref{eq1} and FedAvg is used for aggregation i.e, FedAvg+MAN. 
 Supppose the loss functions $f_k(\mathbf{w})$ in Eq.~\ref{eq1} satisfies the below assumptions as mentioned in~\cite{karimireddy2020scaffold}. 
\begin{assumption}
The loss functions $\text{f}_k$ are Lipschiltz smooth, i.e., $ {\lVert \nabla{f_k(\mathbf{x})} - \nabla{f_k(\mathbf{y})} \rVert} \leq  {\beta} {\lVert\mathbf{x} - \mathbf{y} \rVert}$.
\label{a1_asmp}
\end{assumption}

\begin{assumption}
  $ \frac{1}{n} \sum_{k\in [n]} {\lVert\nabla{f_{k}(\mathbf{w}}) \rVert}^2  \leq G^2 + B^2{\lVert \nabla{f(\mathbf{w})} \rVert}^2$,where $f(\mathbf{w}) = \frac{1}{n} \sum_{k\in [n]} {f_{k}(\mathbf{w}}) $.This is referred to bounded gradient dissimilarity assumption,
  \label{a2_asmp}
\end{assumption}

\begin{assumption}
let $ \mathbb{E}{\lVert f_k(\mathbf{w},(x,y)) - f_k(\mathbf{w})\rVert} \leq \sigma^2$, for all $k$  and $\mathbf{w}$. Here  $f_k(\mathbf{w},(x,y))$ is loss evaluated on the sample $(x,y)$ and $f_k(\mathbf{w})$ is expectation across the samples.
This is a bounded variance assumption. 
\label{a3_asmp}
\end{assumption}
We then have the following proposition.
\begin{prop}
Theorem \MakeUppercase{\romannumeral 5} of~\cite{karimireddy2020scaffold} in Appendix D.2:  let $\mathbf{w}^* = \underset{\mathbf{w}}{\arg\min} \ \tilde{l}(\mathbf{w}) $, the global step-size be $\alpha_g$ and the local step-size be $\alpha_l$. FedAvg+MAN algorithm will have contracting gradients. If Initial model is $\mathbf{w}^0$, $F = \tilde{l}(\mathbf{w}^0)-\tilde{l}(\mathbf{w}^*)$ and for constant $M$, then in $R$ rounds, the model $\mathbf{w}^R$ satisfies
$\mathbb{E}[{\lVert \nabla{{f}(\mathbf{w}^R)} \rVert}^2] \leq 
{O({{\beta M \sqrt{F}} \over {\sqrt{RLS}} } + 
{{\beta^{1/3}(FG)^{2/3} )} \over {(R+1)^{2/3}} } + 
{{\beta B^2 F} \over {R}})
}$.
\label{scaffold_prop}
\end{prop}
The above proposition states that the FedAvg+MAN algorithm requires $\mathcal{O}(\frac{1}{\epsilon^2})$ communication rounds to make the average gradients of the global model smaller, i.e., $\mathbb{E}[{\lVert \nabla{\tilde{l}(\mathbf{w}^R)} \rVert}^2] \leq \epsilon$. Similar guarantees can be given when we use other FL algorithms with the proposed MAN.

\begin{table*}[htp]
\centering
\caption{
Accuracy and communication round comparisons are presented as follows: Accuracy is given in the format of mean ± standard deviation, accompanied by the number of communication rounds. These rounds are required to achieve $50\%$ accuracy for CIFAR-100 and $28\%$ accuracy for Tiny-ImageNet. The experiments are repeated for three initializations, and the mean and standard deviation of accuracy are reported. The performance of different algorithms is shown on CIFAR-100 \& Tiny-ImageNet with and without the MAN regularizer. Accuracy values are reported after 500 communication rounds. On CIFAR-100 FedDC attains accuracy of $52.03\%$  while FedDC+MAN achieves $55.21\%$ accuracy. FedDC reaches $50\%$ accuracy in 294 rounds, whereas FedDC+MAN in just 144 rounds. The utilization of the proposed MAN regularizer clearly enhances the performance and reduces communication rounds for all algorithms.
}
\scalebox{0.76}{
\begin{tabular}{cccc|ccc|}
\toprule
\multirow{3}{*}{Algorithm}                   & \multicolumn{3}{|c}{CIFAR-100}                                                                                        
                                             & \multicolumn{3}{c}{Tiny-ImageNet}                                                    
                                             \\ \cline{2-7} 
                                             & \multicolumn{2}{|c|}{non-iid}                                              
                                             & \multicolumn{1}{c|}{\multirow{2}{*}{iid}} 
                                             & \multicolumn{2}{c|}{non-iid}                                              
                                             & \multicolumn{1}{c}{\multirow{2}{*}{iid}} \\ \cline{2-3} \cline{5-6}
                                             & \multicolumn{1}{|c}{$\delta= 0.3$}  
                                             & \multicolumn{1}{c|}{$\delta= 0.6$}  
                                             & \multicolumn{1}{l|}{}                     
                                             & \multicolumn{1}{c}{$\delta= 0.3$}  
                                             & \multicolumn{1}{c|}{$\delta= 0.6$}  
                                             & \multicolumn{1}{l}{}                     \\ \midrule
FedAvg                                       
& \multicolumn{1}{l}{40.90 $_{\pm{0.62}}$ (500+) }          
& \multicolumn{1}{l}{40.39 $_{\pm{0.59}}$ (500+)}          
& \multicolumn{1}{l|} {39.40 $_{\pm{0.84}}$ (500+)}         
& \multicolumn{1}{l}{25.35 $_{\pm{1.16}}$ (500+)}          
& \multicolumn{1}{l}{24.41 $_{\pm{0.41}}$ (500+)}          
&  \multicolumn{1}{l} {23.75 $_{\pm{0.99}}$ (500+)} \\

FedAvg+MAN (\textbf{Ours})                       
& \multicolumn{1}{l}{\textbf{52.00} $_{\pm{0.36}}$ (\textbf{206})} 
& \multicolumn{1}{l}{\textbf{52.42} $_{\pm{0.23}}$ (\textbf{210})} 
&  \multicolumn{1}{l|}{\textbf{52.59} $_{\pm{0.25}}$ (\textbf{224})} 
& \multicolumn{1}{l}{\textbf{28.09} $_{\pm{0.26}}$ (\textbf{437})} 
& \multicolumn{1}{l}{\textbf{28.90} $_{\pm{0.21}}$ (\textbf{194})} 
& \multicolumn{1}{l}{\textbf{29.11} $_{\pm{0.12}}$ (\textbf{182})}           \\ \midrule
FedSAM

& \multicolumn{1}{l}{43.44 $_{\pm{0.11}}$ (500+)}          
& \multicolumn{1}{l}{43.36 $_{\pm{0.24}}$ (500+)}          
& \multicolumn{1}{l|}{41.31 $_{\pm{0.27}}$ (500+)}           
& \multicolumn{1}{l}{26.23 $_{\pm{0.68}}$ (500+)}          
& \multicolumn{1}{l}{26.04 $_{\pm{0.20}}$ (500+)}          
& \multicolumn{1}{l} {23.97 $_{\pm{0.83}}$ (500+)}                                     \\ 
FedSAM+MAN (\textbf{Ours})                       
& \multicolumn{1}{l}{\textbf{51.59} $_{\pm{0.48}}$ (\textbf{326})}          
& \multicolumn{1}{l}{\textbf{52.62} $_{\pm{0.40}}$ (\textbf{281})}          
& \multicolumn{1}{l|}{\textbf{52.85} $_{\pm{0.13}}$ (\textbf{301})}     
& \multicolumn{1}{l}{\textbf{32.16} $_{\pm{0.20}}$ (\textbf{104})}          
& \multicolumn{1}{l}{\textbf{32.60} $_{\pm{0.92}}$ (\textbf{87})}          
& \multicolumn{1}{l}{\textbf{31.40} $_{\pm{0.30}}$ (\textbf{82})}   
\\ 
\midrule

FedASAM

& \multicolumn{1}{l}{46.00 $_{\pm{0.10}}$ (500+)}          
& \multicolumn{1}{l}{45.48 $_{\pm{0.08}}$ (500+)}          
& \multicolumn{1}{l|}{44.18 $_{\pm{0.61}}$ (500+)}           
& \multicolumn{1}{l}{27.50 $_{\pm{0.09}}$ (500+)}          
& \multicolumn{1}{l}{27.05 $_{\pm{0.16}}$ (500+)}          
& \multicolumn{1}{l} {23.96 $_{\pm{0.43}}$ (500+)}                                      \\ 
FedASAM+MAN (\textbf{Ours})                       
& \multicolumn{1}{l}{\textbf{51.23} $_{\pm{0.20}}$ (\textbf{313}) }          
& \multicolumn{1}{l}{\textbf{51.89} $_{\pm{0.09}}$ (\textbf{351})}          
& \multicolumn{1}{l|}{\textbf{52.41} $_{\pm{0.43}}$ (\textbf{296})}     
& \multicolumn{1}{l}{\textbf{32.50} $_{\pm{0.05}}$ (\textbf{140})}          
& \multicolumn{1}{l}{\textbf{32.41} $_{\pm{0.32}}$ (\textbf{136})}          
& \multicolumn{1}{l}{\textbf{31.70} $_{\pm{0.81}}$ (\textbf{93})}   
\\ 
\midrule
FedDyn

& \multicolumn{1}{l}{49.29 $_{\pm{0.30}}$ (500+)}          
& \multicolumn{1}{l}{49.91 $_{\pm{0.41}}$ (486)}          
& \multicolumn{1}{l|}{50.04 $_{\pm{0.22}}$ (500+)}            
& \multicolumn{1}{l}{29.23 $_{\pm{0.06}}$ (295)}          
& \multicolumn{1}{l}{28.99 $_{\pm{0.55}}$ (308)}          
& \multicolumn{1}{l} {29.41 $_{\pm{1.33}}$ (350)}                                     \\ 
FedDyn+MAN (\textbf{Ours})                        
& \multicolumn{1}{l}{\textbf{55.27} $_{\pm{0.12}}$ (\textbf{145})}          
& \multicolumn{1}{l}{\textbf{55.63} $_{\pm{0.37}}$ (\textbf{143})}          
& \multicolumn{1}{l|}{\textbf{55.83} $_{\pm{0.56}}$ (\textbf{157})}     
& \multicolumn{1}{l}{\textbf{32.00} $_{\pm{0.57}}$ (\textbf{132})}          
& \multicolumn{1}{l}{\textbf{32.44} $_{\pm{0.27}}$ (\textbf{110})}          
& \multicolumn{1}{l}{\textbf{32.31} $_{\pm{0.38}}$ (\textbf{108})}                       
\\ \midrule
FedDC                                    
& \multicolumn{1}{l}{52.02 $_{\pm{0.79}}$ (294)}          
& \multicolumn{1}{l}{52.64 $_{\pm{0.24}}$ (304)}          
& \multicolumn{1}{l|}{53.25 $_{\pm{0.86}}$ (289)}         
& \multicolumn{1}{l}{31.44 $_{\pm{0.43}}$ (170)}         
& \multicolumn{1}{l}{31.42 $_{\pm{0.36}}$ (193)}          
&  \multicolumn{1}{l}{31.21 $_{\pm{0.43}}$ (201)}                                     \\ 
FedDC + MAN (\textbf{Ours}) 
& \multicolumn{1}{l}{\textbf{55.21} $_{\pm{0.32}}$ (\textbf{144}) }    
& \multicolumn{1}{l}{\textbf{55.40} $_{\pm{0.30}}$ (\textbf{129})}   
& \multicolumn{1}{l|}{\textbf{56.77} $_{\pm{0.31}}$ (\textbf{156})}    
& \multicolumn{1}{l}{\textbf{35.70} $_{\pm{0.21}}$ (\textbf{73})}    
& \multicolumn{1}{l}{\textbf{36.07} $_{\pm{0.23}}$ (\textbf{66})}    
& \multicolumn{1}{l}{\textbf{36.53} $_{\pm{0.03}}$ (\textbf{61})} 

\\ \midrule
FedSpeed                                 
& \multicolumn{1}{l}{50.95 $_{\pm{0.02}}$ (392)}          
& \multicolumn{1}{l}{51.33 $_{\pm{0.17}}$ (390)}          
& \multicolumn{1}{l|}{50.95 $_{\pm{0.51}}$ (439)}         
& \multicolumn{1}{l}{31.12 $_{\pm{0.61}}$ (211)}         
& \multicolumn{1}{l}{31.10 $_{\pm{0.27}}$ (228)}          
&  \multicolumn{1}{l}{29.65 $_{\pm{0.11}}$ (363)}                                     \\ 
FedSpeed + MAN (\textbf{Ours}) 
& \multicolumn{1}{l}{\textbf{55.23} $_{\pm{0.15}}$  (\textbf{163})}    
& \multicolumn{1}{l}{\textbf{55.84} $_{\pm{0.11}}$ (\textbf{153})}   
& \multicolumn{1}{l|}{\textbf{55.89} $_{\pm{0.41}}$ (\textbf{163})}    
& \multicolumn{1}{l}{\textbf{34.32} $_{\pm{0.63}}$ (\textbf{108})}    
& \multicolumn{1}{l}{\textbf{35.49} $_{\pm{0.07}}$ (\textbf{78})}    
& \multicolumn{1}{l}{\textbf{33.02} $_{\pm{0.55}}$ (\textbf{98})} 
\\ \bottomrule
\end{tabular}
}

\label{table_acc}
\end{table*}
\section{Experiments}
We perform experiments under both the non-iid and iid setups. For non-iid, we experiment on the label imbalance, whose experimental setups are described below.
We adopted the experimental setup described in~\cite{Gao_2022_CVPR,acar2021federated}, using CIFAR-100~\cite{krizhevsky2009learning} and Tiny-ImageNet~\cite{le2015tiny} datasets. The model specifications are available in Sec. \textcolor{red}{3} of the suppl material. Both iid and non-iid data partitioning were tested in our experiments, with $100$ clients participating and $10\%$ of them being selected at random for each communication round. Each client was allocated the same number of samples, and the accuracy was measured at the end of $500$ communication rounds. Non-iid data was generated using the Dirichlet distribution $Dir(\delta)$, following the approach in~\cite{acar2021federated}. A label distribution vector was sampled for each client from the Dirichlet distribution; the entries of the vector are non-negative and summed to $1$. The value of $\delta$ controls the degree of non-iid data, with lower values resulting in higher label imbalances. The effect of $\delta$ on client distribution is presented in Sec.\textcolor{red}{6} of supplementary. We also conducted a sensitivity analysis of the hyper-parameter $\zeta$ on global model accuracy, as described in Sec. \textcolor{red}{4} of supplementary. The learning rate was set to $0.1$ with a decay of $0.998$ per communication round, batch size of $50$ and $5$ epochs per round. 
All these hyper-parameter settings are consistent with the~\cite{acar2021federated, Gao_2022_CVPR}. We set the hyper-parameter $\zeta$ to $0.15$ for CIFAR-100 and $0.1$ for Tiny-ImageNet, and more details are in Sec.\textcolor{red}{7} of supplementary. 
We have also performed experiments on the CIFAR-10 dataset. The results are given in Sec \textcolor{red}{5.2} of the suppl material. 
\begin{figure*}[htp]
  \centering
  \subfloat[$\delta = 0.3$]{
  \includegraphics[scale=0.39]{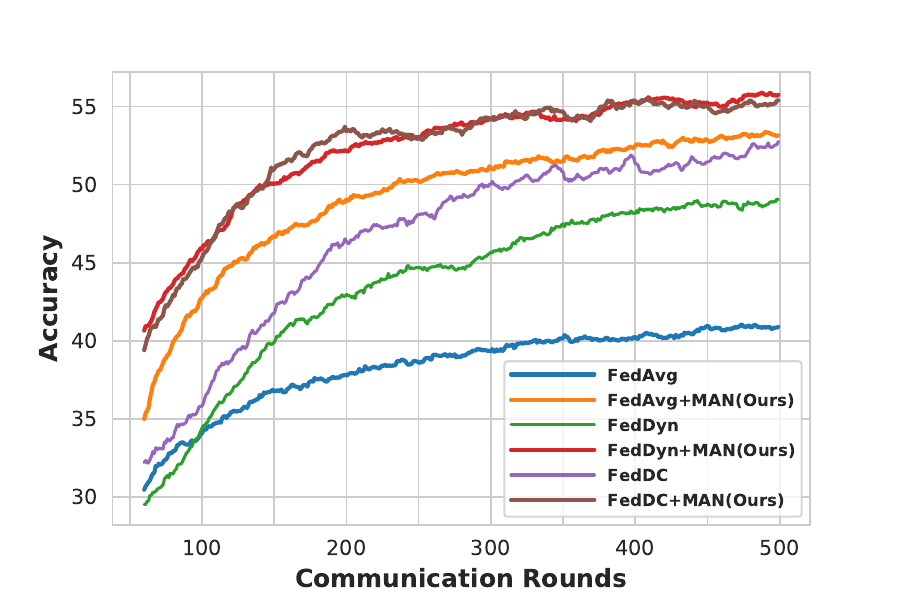}}\hspace{-2.em}
  \subfloat[$\delta = 0.6$]{
  \includegraphics[scale=0.39]{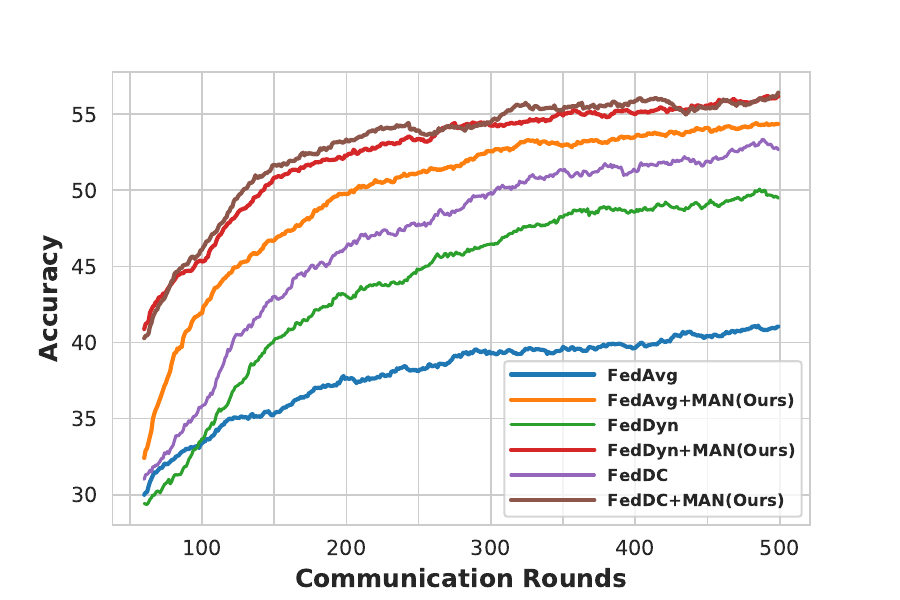}}\hspace{-2.em}
  \subfloat[iid]{
  \includegraphics[scale=0.39]{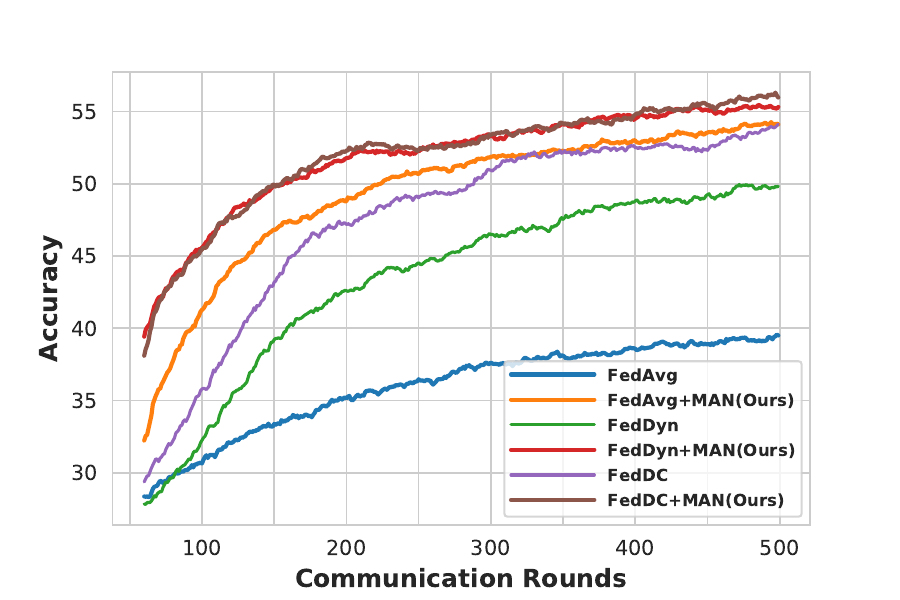}
  }
  \caption{Convergence Comparison on CIFAR-100: We compare performance of the algorithms FedAvg, FedDyn, FedDC and the proposed FedAvg+MAN, FedDyn+MAN and FedDC+MAN for 500 communication rounds. It can be clearly seen that  proposed approach significantly improves the existing algorithms across the communication rounds. }
  \label{cifar100_perf}
\end{figure*}
\begin{figure*}[htp]
  \centering
  \subfloat[$\delta = 0.3$]{
  \includegraphics[scale=0.39]{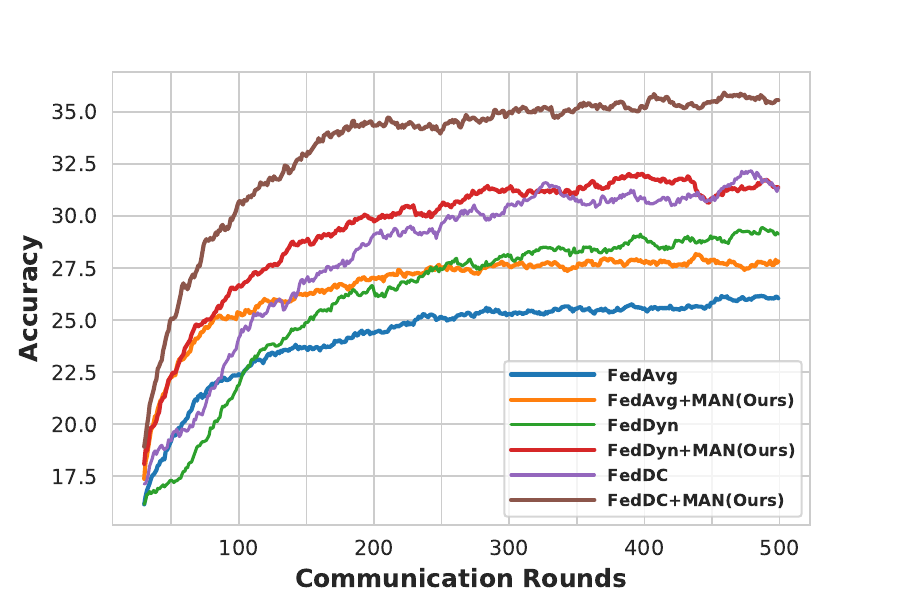}} \hspace{-2.em}
  \subfloat[$\delta = 0.6$]{
  \includegraphics[scale=0.39]{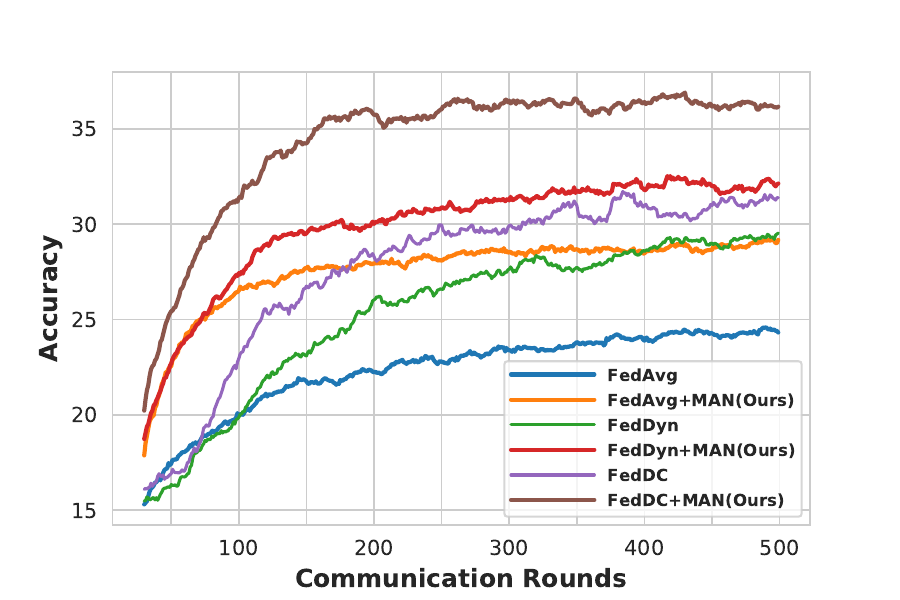}}
  \hspace{-2.em}
  \subfloat[iid]{
  \includegraphics[scale=0.39]{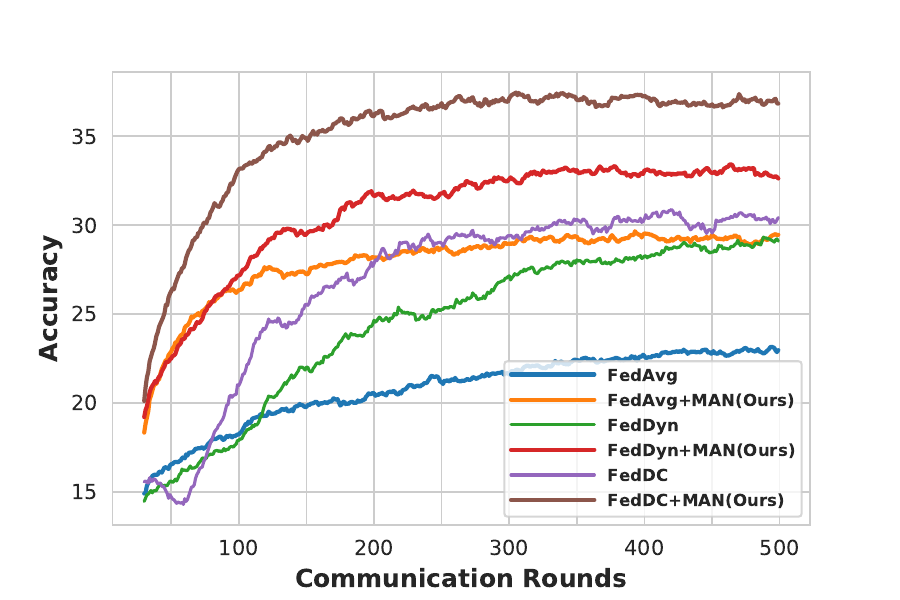}
  }
  \caption{Convergence Comparison on Tiny-ImageNet: We compare the performance of the algorithms FedAvg, FedDyn, FedDC and the proposed FedAvg+MAN, FedDyn+MAN, and FedDC+MAN for 500 communication rounds. It can be clearly seen that the proposed approach significantly improves the existing algorithms.}
  \label{TinyImageNet_perf}
\end{figure*}%
\vspace{-0.1in}
\section{Results and Discussion}
\label{result_sec}
In Table \ref{table_acc}, we summarize our results. We refer to FedAvg+MAN, FedDC+MAN, FedDyn+MAN, FedDC+MAN, FedSpeed+MAN and FedSAM/ASAM+MAN when we use the algorithms FedAvg, FedDC, FedDyn, FedSpeed, and FedSAM/ASAM respectively, along with our flatness inducing regularizer MAN ( see Sec.\textcolor{red}{8} of supplementary for algorithm details). From Table \ref{table_acc}, we observe that on CIFAR-100, FedDC attains an accuracy of $52.02\%$ and it attains the $50\%$ accuracy is 294 rounds for $\delta=0.3$ whereas FedDC+MAN attains the accuracy of $55.21\%$ and it attained $50\%$ accuracy in just 144 rounds for the same $\delta=0.3$, thus leading to an improved performance by $3.2\%$ and saving $150$ rounds of communication.  
Similarly we can see that FedDC+MAN improves FedDC by $2.76\%$, $3.5\%$ for  $\delta = 0.6$ and iid data partitions, respectively. It also saves the $175$ and $133$ rounds of communication on $\delta = 0.6$ and iid data partitions, respectively.  
For the Tiny-ImageNet dataset, we improve the performance of FedDC by $4.2\%$, $4.5\%$, $5.1\%$ for $\delta = 0.3$, $\delta = 0.6$ and iid data partitions respectively. FedDC+MAN also saves $97$,$127$ and $140$ rounds of communication compared to FedDC $\delta = 0.3$, $\delta = 0.6$ and iid data partitions respectively. Similar improvements can be seen for FedDyn, FedSpeed, FedSAM/ASAM as well. 
To get smoother estimates, we follow the protocol of~\cite{acar2021federated}, where we take the average of all the client models for reporting accuracy. 
In Figures~\ref{cifar100_perf} and~\ref{TinyImageNet_perf}, the performance of the algorithms FedAvg, FedDyn, and FedDC, are compared with flatness-constrained versions FedAvg+MAN, FedDyn+MAN, and FedDC+MAN. Clearly, our flatness-constrained version of algorithms significantly performs better. Figures~\ref{cifar100_perf} and~\ref{TinyImageNet_perf} are generated for a single training seed. The accuracy vs communication plots for FedSAM/ASAM, and FedSpeed with and without MAN are given in the Sec. \textcolor{red}{5.1} of the supplementary.
\begin{figure}[htp]
 \centering
  \begin{subfigure}[b]{0.29\textwidth}
  \centering
   \includegraphics[width=0.9\linewidth]{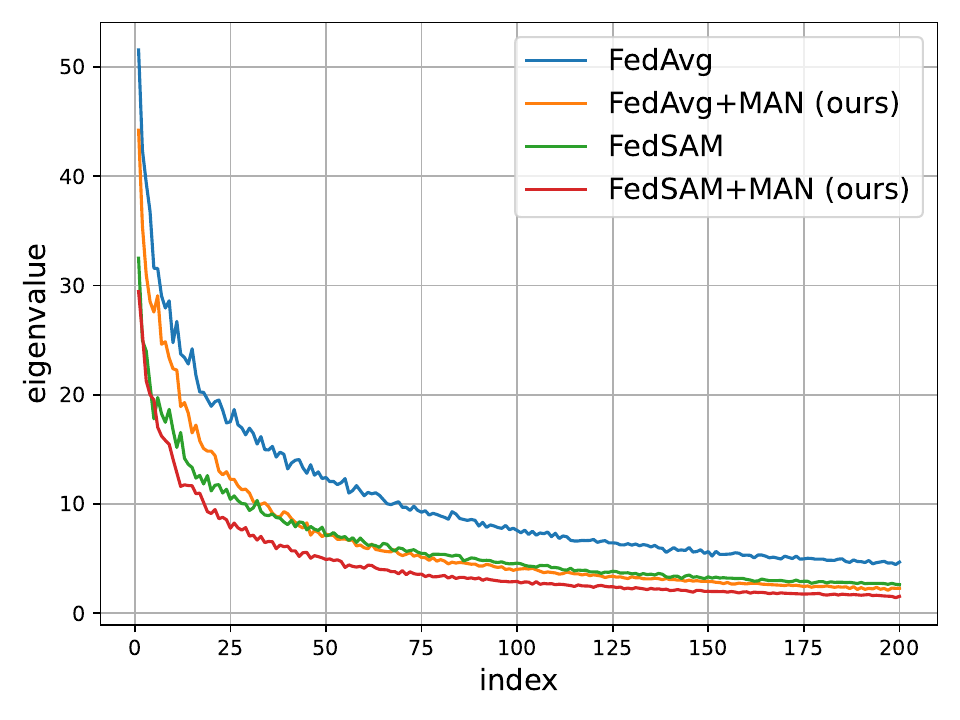}
   \caption{Top 200 eigenvalues comparison}
   \label{fig:Ng1} 
\end{subfigure}

\begin{subfigure}[b]{0.29\textwidth}
  \centering
   \includegraphics[width=0.9\linewidth]{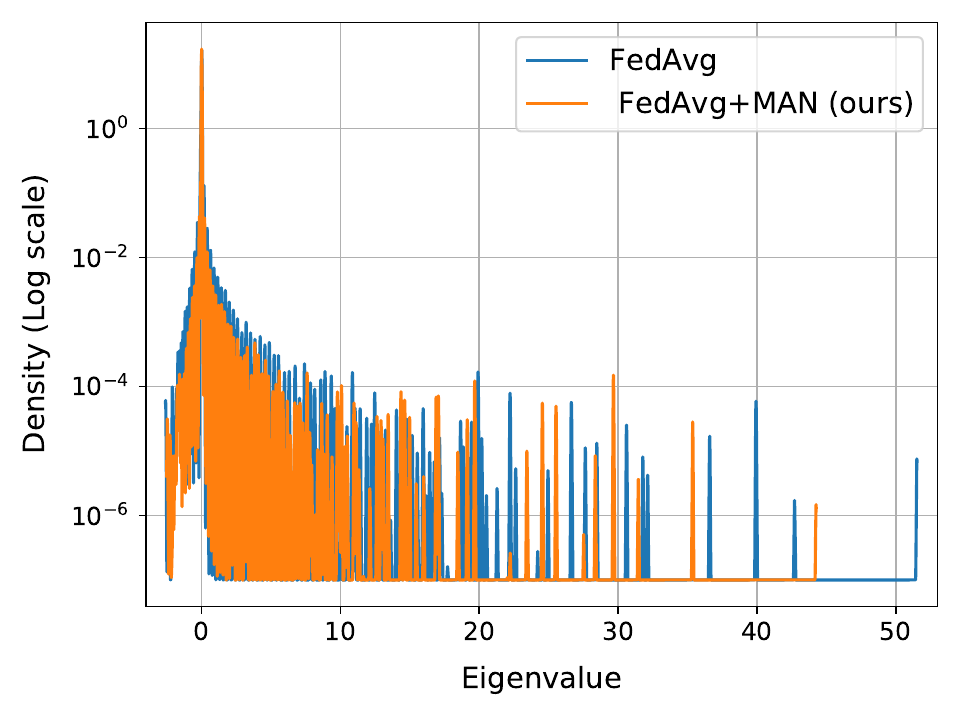}
   \caption{Eigen spectral density}
   \label{fig:Ng2}
\end{subfigure}
  \caption{Figure~\ref{fig:Ng1} shows the comparison of top 200 eigenvalues of FedAvg/FedAvg+MAN and FedSAM/FedSAM+MAN. It can be seen MAN regularizer reduces the top eigenvalues. This explains the reduction of the trace observed. From the Figure~\ref{fig:Ng2} we see negative eigenvalues contributes little to the trace.  }
  \label{eig_density}
\end{figure}

\begin{table}[htp]
\centering
\caption{Comparison of top eigenvalues and trace of the algorithms with and without MAN regularizer, lower values are better. We can observe that by
augmenting MAN regularization i.e. FedAvg+MAN, FedSAM+MAN, etc.,  we obtain lower trace and lower top eigenvalues, which is indicative of flat minimum, and hence it attains better accuracy.}
\scalebox{0.65}{
\begin{tabular}{l|cccc}
\toprule
                         & \multicolumn{4}{c}{CIFAR-100}                                                                                                                       \\ \cline{2-5} 
                         & \multicolumn{2}{c|}{$\delta = 0.3$}                                   & \multicolumn{2}{c}{$\delta = 0.6$}                                                \\ \cline{2-5} 
\multirow{-3}{*}{Method} & \multicolumn{1}{l}{\begin{tabular}[c]{@{}c@{}}Top \\ eigenvalue\end{tabular}} & \multicolumn{1}{l|}{Trace}   & \multicolumn{1}{l}{\begin{tabular}[c]{@{}c@{}}Top \\ eigenvalue\end{tabular}}            & \multicolumn{1}{l}{Trace}     \\ \midrule
FedAvg                   
& \multicolumn{1}{c}{51.49}          
& \multicolumn{1}{c|}{8744} 
& \multicolumn{1}{c}{53.39}                     
& 9056                        
\\ 
FedAvg+MAN               & \multicolumn{1}{c}{\textbf{43.80}}          
& \multicolumn{1}{c|}{\textbf{4397}}
& \multicolumn{1}{c}{\textbf{42.00}} 
& {\textbf{4747}} 
\\ \hline
FedSAM                   
& \multicolumn{1}{c}{32.46}          
& \multicolumn{1}{c|}{4909} 
& \multicolumn{1}{c}{35.32}                     
& 5160                        \\ 
FedSAM+MAN               
& \multicolumn{1}{c}{\textbf{29.29}}          
& \multicolumn{1}{c|}{\textbf{2709}} 
& \multicolumn{1}{c}{\textbf{30.05}}                     
& \textbf{2918}   
\\ \hline
FedDyn                    
& \multicolumn{1}{c}{51.03}          
& \multicolumn{1}{c|}{6400} 
& \multicolumn{1}{c}{47.03}                     
& 6717                       \\ 
FedDyn+MAN                
& \multicolumn{1}{c}{\textbf{44.84}}          
& \multicolumn{1}{c|}{\textbf{3964}} 
& \multicolumn{1}{c}{\textbf{38.44}}                     
& \textbf{3966} 
\\ \hline
FedDC                    
& \multicolumn{1}{c}{35.13}          
& \multicolumn{1}{c|}{3578} 
& \multicolumn{1}{c}{35.13}                     
& 3578                       \\ 
FedDC+MAN                
& \multicolumn{1}{c}{\textbf{32.99}}          
& \multicolumn{1}{c|}{\textbf{2974}} 
& \multicolumn{1}{c}{\textbf{32.99}}                     
& \textbf{2974}                        \\ \bottomrule
\end{tabular}
}

\label{sam_analysis}
\end{table}

\subsection{Empirical Analysis of Hessian}
\label{emp_hess_analysis}
In Fig~\ref{eig_density} and Table~\ref{sam_analysis}, we perform the Hessian Analysis of the proposed MAN regularizer on the CIFAR-100 and compare it against the baseline algorithms. We observe empirically that when our regularizer MAN is combined with FedAvg, FedDyn and FedDC, it attains a flatter minimum, which is quantified by reduction in top eigenvalue and the trace of the Hessian. The top eigenvalue is a key indicator of better generalization~\cite{yao2020pyhessian,Mendieta_2022_CVPR,keskar2017on}. Due to space constraints, More results of Hessian analysis on the remaining algorithms are presented in Sec \textcolor{red}{5.3} of supplementary. To better understand the reduction of the trace, we have plotted the top 200 eigenvalues for FedAvg, FedAvg+MAN, FedSAM and FedSAM+MAN. We can see that MAN regularizer not just reduces the top eigenvalue but also reduces the other eigenvalues as well. This is the reason for reduction of the trace. It is important to emphasize that the trace can be lower when negative eigenvalues also contribute significantly. Since we are evaluating the trace at the convergence, we observe that the contribution of negative eigenvalues is negligible, and the trace is dominated by positive eigenvalues. The eigen spectral density Fig~\ref{fig:Ng2} confirms this observation. We can also observe that FedAvg+MAN has lower trace compared to FedAvg+MAN from the table~\ref{sam_analysis}. This can be observed from Fig~\ref{fig:Ng1}, where after the first ~$50$ eigen values FedAvg+MAN has lower eigenvalues consistently. This suggests that MAN regularizer reduces many eigenvalues which is a cause for it's better generalization.

\vspace{-0.3in}.
\section{ Analyzing Computational Cost}
We now analyze the total computation cost incorporating the proposed regularizer. We analyze the total number of multiplications in a forward pass at a specific layer to measure the computation. Let $N_b$, $C_i$, $H_i$, $W_i$ denote the batch size, input channels, height, and width of the activation map, which is fed to a convolutional layer. We assume the kernel size to be $C_i\times K \times K$ and a number of such filters to be $C_o$. Thus the output is represented by $N_{b}$, $C_o$, $H_o$, $W_o$. Note that batch size $N_b$ remains the same. 
To compute the single entry in the output activation map, we need $C_iK^2$ multiplications. For the entire spatial dimension, we need $H_oW_oC_iK^2$, and for all the output channels $C_o$, we need $C_oH_oW_oC_iK^2$ and for the batch of $N_b$ we have 
$N_bC_oH_oW_oC_iK^2$ operations.
\begin{table}[htp]
\centering

\caption{computations without the regularizer and only regularizer.}
\vspace{-2mm}
\scalebox{0.7}{
\begin{tabular}{l|l|l}
\hline
Forward Computation    & CONV Layer & FC Layer \\ 
\hline
Without Regularizer                     & $N_bC_oH_oW_oC_iK^2$            & $N_bd_id_o$  
\\ \hline
Only Regularizer                 & $N_bH_oW_oC_o$                  & $N_bd_o$  
\\ \hline
\end{tabular}
}

\label{computation_cost}
\end{table}
Similarly, let $d_i$ be the input features and $d_o$ be the output feature dimension for a fully connected layer in the network. We need a total of $N_bd_id_o$ operations for fully connected layers. This is summarized in Table~\ref{computation_cost}.
The total cost  for the convolutional (CONV) layer denoted by $TC_{conv} $ with our regularizer is, $TC_{conv} =  N_bC_oH_oW_o(1+C_iK^2$). Typicaly the $C_iK^2 >>1$, this would mean that $TC_{conv}\approx N_bC_oH_oW_oC_iK^2$, which is same as the cost without regularizer from Table~\ref{computation_cost}.  
The total cost for fully-connected (FC) layer including our regularizer is denoted by $TC_{fc}$ is ($TC_{fc} = N_bd_o(1+d_i)$). Since $d_i >>1$ we approximate the total cost for FC layer as $TC_{fc} \approx N_bd_od_i$. This shows that our regularizer incurs negligible cost compared to forward operations without a regularizer for both the convolutional and fully connected layers. Similar analysis can be done for backward pass operations. Our regularizer incurs half the computation required by SAM-based methods, as the SAM based methods require gradient ascent followed by gradient descent. 
\section{Conclusion}
\vspace{-0.06in}
 In this paper, we introduce a novel problem formulation in the context of federated learning that includes flatness constraints. By doing so, we demonstrate that by having the client models converge to a flat minimum, the global model also converges to a flat minimum. Additionally, we have simplified the problem at the client level by minimizing activation norms, and we have shown theoretically that this approach can minimize the layer-wise top eigenvalues of the Hessian of the client's loss, which, in turn, leads to minimizing the top eigenvalue of the overall Hessian of the loss. Our proposed methodology can be seamlessly integrated atop existing FL algorithms. In particular, we have integrated our flatness constraint objective on top of popular FL methods such as FedAvg, FedDyn, FedDC and SAM-based methods such as FedSAM/ASAM and FedSpeed. We have shown that this can significantly improve the performance of these baselines. Furthermore, we have demonstrated that our method incurs negligible computation cost after incorporating our regularizer. Our work presents a promising new approach for incorporating flatness constraints as a computationally efficient regularizer into FL algorithms that can lead to better generalization performance. Our work can serve as a beginning for the development of computationally efficient algorithms for inducing flatness constraints.
 
 {\small \noindent \textbf{Acknowledgement}: M. Yashwanth and Harsh Rangwani are supported by the Prime Minister’s Research Fellowship (PMRF). We are thankful for the support provided.} 

{\small
\bibliographystyle{ieee_fullname}
\bibliography{egbib}
}
\clearpage

\twocolumn[
\begin{center}
    {\LARGE \bf  Supplementary material for ``Minimizing Layerwise Activation Norm Improves Generalization in Federated Learning"}
\end{center}
\vspace{2em}
]

\setcounter{section}{0}


\section{Notations and Preliminaries}
We 
describe the necessary preliminaries and notation below. 
\subsection{Notations}
By default, we assume the notation ${\lVert . \rVert}$ for 2-norm. ${\lVert . \rVert}_{F}$ is the Frobenious norm. ${\lVert . \rVert}_{2}$ is the usual euclidean norm. $\lambda_{max}(\textbf{A})$ is maximum eigenvalue of matrix \textbf{A}. $\odot$ denotes the element-wise product of two matrices or vectors. $\otimes$ denotes the Kronecker product of two matrices or vectors. $\nabla_{\textbf{b}}{\textbf{a}}$ denotes the Jacobian of $\textbf{a}$ w.r.t \textbf{b}.
eig(\textbf{A}) denotes the eigenvalues of matrix \textbf{A}.
$vec(\textbf{A})$ denotes the vectorization operation. If $\textbf{A} \in \mathbb{R}^{n_1 \times n_2}$, then $vec(\textbf{A}) \in \mathbb{R}^{n_1n_2 \times 1}$($n_2$ columns of \textbf{A} are stacked one after other).
$\mathbf{H}_{\mathbf{W}}{(\mathcal{L})}$ is Hessian of Loss $\mathcal{L}$ with respect to parameters $\mathbf{W}$. $\lambda(\mathbf{H})$ denotes eigenvalue of $\mathbf{H}$. $\mathcal{O}$ denotes the usual Big-$\mathbf{O}$ notation.

\subsection{Preliminaries} 




If $\textbf{x} \in \mathbb{R}^{d\times 1}$
\begin{equation}
eig(\textbf{x} \textbf{x}^{\top}) = {\lVert \textbf{x} \rVert}_2^2
\label{out_eig_eq}    
\end{equation}

\begin{equation}
(\textbf{A} \otimes \textbf{B})^{\top} = \textbf{A}^{\top} \otimes \textbf{B}^{\top}
\label{pl2}
\end{equation}

If \textbf{A},\textbf{B},\textbf{C},\textbf{D} matrices of compatible dimensions then the following holds

\begin{equation}
(\textbf{A} \otimes \textbf{B})(\textbf{C} \otimes \textbf{D}) = \textbf{A}\textbf{C} \otimes \textbf{B}\textbf{D}
\label{pl3}
\end{equation}

\begin{equation}
(\textbf{A} \otimes \textbf{B})\textbf{C} = \textbf{A}\textbf{C} \otimes \textbf{B}
\label{pl3}
\end{equation}

\begin{equation}
eig(\textbf{A} \otimes \textbf{B}) = eig(\textbf{A}) \otimes eig(\textbf{B}).  
\label{pl4}
\end{equation}

For any symmetric matrices $\textbf{S}_1$ and $\textbf{S}_2$, the following holds.

\begin{equation}
\lambda_{max}(\textbf{S}_1 + \textbf{S}_2) \leq  \lambda_{max}(\textbf{S}_1) + \lambda_{max}(\textbf{S}_2)
\label{pl5}
\end{equation}
We now provide proof of the theorems in the next section.
\section{Proof for the Theorems in the main paper}
\begin{theorem}
If $\mathbf{H}_{ll} \in R^{d_l}$  denotes the layer $l$ Hessian and $\mathbf{H} \in R^{d}$ denotes the over all Hessian and $\sum_{l=1}^{L}{d_l} = d$, where $L$ is the total number of layers. If the Hessian entries are bounded above we then have the following result.
$\lambda(\mathbf{H}) \in \cup_{l=1}^{L} [ \lambda_{min}(H_{ll}) - \mathcal{O}({max(d_l,d-d_l)}), \lambda_{max}(H_{ll}) + \mathcal{O}({max(d_l,d-d_l)})] $
\label{hess_theorem0}
\end{theorem}
\begin{proof}

Let $\mathbf{x} \in R^{d \times 1}$ be the eigen vector of $\mathbf{H}$ be the Hessian matrix and it is symmetric partitioned with $\mathbf{H}_{ij} \in R^{d_i \times d_j}$. The block diagonal matrices are the ones with $i=j$ and there are $L$ such matrices along the diagonal. We also assume that $\mathbf{x}$ is partitioned into $L$ vectors as 
$\mathbf{x} = {[ {\mathbf{x}_1}^{\intercal},{\mathbf{x}_2}^{\intercal},..., {\mathbf{x}_L}^{\intercal} ]}^{\intercal}$ where $\mathbf{x}_j \in R^{d_j \times 1}$.

Since $\mathbf{x}$ is assumed to be the eigenvector and $\lambda$ be associated eigenvalue  we have the following
\begin{equation}
\lambda\mathbf{x}_l = \mathbf{H}_{ll}\mathbf{x}_l + \sum_{j=1,j\neq l }^{L}{\mathbf{H}_{lj}} \mathbf{x}_j
\label{th_eig_eq1}
\end{equation}

Here $\mathbf{x}_l$ is chosen such that $\lVert \mathbf{x}_l \rVert \geq \lVert \mathbf{x}_i \rVert$ for all $i$. Throughout the proof we mean  $\lVert . \rVert$ as 2-norm. 

Taking the norm on Eq.~\ref{th_eig_eq1} we get the following

\begin{dmath}
\lVert(\lambda \mathbf{I} - \mathbf{H}_{ll}) \mathbf{x}_l \rVert =   \sum_{j=1,j\neq l }^{L} {\lVert {{\mathbf{H}_{lj}} \mathbf{x}_j } \rVert} 
\leq 
  \sum_{j=1,j\neq l }^{L} {\lVert {{\mathbf{H}_{lj}} \rVert \lVert \mathbf{x}_j \rVert }}
\label{th_eig_eq2_}
\end{dmath}

The first inequality is by triangle inequality and the second is by the definition of norm
$\lVert \mathbf{A}\mathbf{x} \rVert \leq \lVert \mathbf{A} \rVert \lVert\mathbf{x} \rVert$.

Dividing the equation~\ref{th_eig_eq2_} by $\lVert \mathbf{x}_l \rVert$ we get the following
\begin{dmath}
\frac{\lVert(\lambda \mathbf{I} - \mathbf{H}_{ll}) \mathbf{x}_l \rVert}{\lVert  \mathbf{x}_l \rVert } \leq 
\sum_{j=1,j\neq l }^{L} {\lVert {{\mathbf{H}_{lj}} \rVert \frac{\lVert \mathbf{x}_j \rVert } {{\lVert  \mathbf{x}_l \rVert}}}}
\leq
\sum_{j=1,j\neq l }^{L} {\lVert {{\mathbf{H}_{lj}} \rVert}}
\label{th_eig_eq2}
\end{dmath}

The second inequality follows as ${\lVert  \mathbf{x}_j \rVert} \leq {\lVert  \mathbf{x}_l \rVert}$.

It is easy to see that 
\begin{equation}
\frac{\lVert(\lambda \mathbf{I} - \mathbf{H}_{ll}) \mathbf{x}_l \rVert}{\lVert  \mathbf{x}_l \rVert }\geq \underset{i}{min}{{\lvert \lambda - \lambda_i(\mathbf{H}_{ll}) \rvert}}
\label{th_eig_eq3}
\end{equation}
as $(\lambda \mathbf{I} - \mathbf{H}_{ll})$ is a symmetric matrix.
Here $\lambda_1(\mathbf{H}_{ll}) \geq \lambda_2(\mathbf{H}_{ll}) 
 ...  \geq \lambda_{d_l}(\mathbf{H}_{ll})$
using Eq~\ref{th_eig_eq2} and Eq~\ref{th_eig_eq3} we get the following. 
\begin{equation}
\underset{i}{min}{{\lvert \lambda - \lambda_i(\mathbf{H}_{ll}) \rvert}} \leq \sum_{j=1,j\neq l }^{L} {\lVert {{\mathbf{H}_{lj}} \rVert}}
\label{th_eig_eq4}
\end{equation}

This implies that any of the following is true 

${{\lvert \lambda - \lambda_1(\mathbf{H}_{ll}) \rvert}} \leq \sum_{j=1,j\neq l }^{L} {\lVert {{\mathbf{H}_{lj}} \rVert}}$ or ${{\lvert \lambda - \lambda_2(\mathbf{H}_{ll}) \rvert}} \leq \sum_{j=1,j\neq l }^{L} {\lVert {{\mathbf{H}_{lj}} \rVert}}$ or the ${{\lvert \lambda - \lambda_{d_l}(\mathbf{H}_{ll}) \rvert}} \leq \sum_{j=1,j\neq l }^{L} {\lVert {{\mathbf{H}_{lj}} \rVert}}$.

Hence, we take the worst-case possibility that contains all the regions i.e., 
\begin{equation}
{{\lambda \leq \lambda_1(\mathbf{H}_{ll})}} +  \sum_{j=1,j\neq l }^{L} {\lVert {{\mathbf{H}_{lj}} \rVert}}
\label{th_eig_eq5}
\end{equation}

\begin{equation}
{{ \lambda }} \geq \lambda_{d_l}(\mathbf{H}_{ll}) -\sum_{j=1,j\neq l }^{L} {\lVert {{\mathbf{H}_{lj}} \rVert}} 
\label{th_eig_eq6}
\end{equation}
Note that $\lambda_1({\mathbf{H}_{ll}}) \triangleq \lambda_{max}({\mathbf{H}_{ll}})$ and $\lambda_{d_l}({\mathbf{H}_{ll}}) \triangleq \lambda_{min}({\mathbf{H}_{ll}})$.

It remains to show that $\sum_{j=1,j\neq l }^{L} {\lVert {{\mathbf{H}_{lj}} \rVert}} \leq \mathcal{O}(max(d_l,d-d_l))$.

\begin{equation}
{\lVert {{\mathbf{H}_{lj}} \rVert}}^2 \leq {\lVert {{\mathbf{H}_{lj}} \rVert}}^2_F \leq {B d_l d_j}
\label{th_eig_eq7}
\end{equation}
where ${\lVert {\mathbf{H}_{lj}} \rVert}_F$ is the Frobenious norm of the matrix.
The first inequality in Eq.~\ref{th_eig_eq7} is because the 2-norm is, at most, the Frobenious norm, and the second inequality is from the assumption of the theorem that each entry is bounded above by $B$.

\begin{dmath}
\sum_{j=1,j\neq l }^{L} {\lVert {{\mathbf{H}_{lj}} \rVert}} \leq 
\sqrt{(L-1)\sum_{j=1,j\neq l }^{L} {\lVert {{\mathbf{H}_{lj}} \rVert}}^2} \leq \sqrt {B(L-1)d_l (d-d_l)}
\leq \mathcal{O}(max(d_l,d-d_l)).
\label{th_eig_eq8}
\end{dmath}

The first inequality uses Cauchy Schwartz inequality, and the second inequality uses Eq.~\ref{th_eig_eq7} and also the fact that $\sum_{i=1}^{i=1}{d_i} = d$ and the third inequality follows by definition of $\mathcal{O}$. From Eq.~\ref{th_eig_eq5}, Eq.~\ref{th_eig_eq6} and Eq.~\ref{th_eig_eq8} we have shown that $\lambda(\mathbf{H}) \in [ \lambda_{min}(H_{ll}) - \mathcal{O}({max(d_l,d-d_l)}), \lambda_{max}(H_{ll}) + \mathcal{O}({max(d_l,d-d_l)})] $.
But the $l$ can be anywhere from $l=1$ to $l=L$ hence we take the union of all the possible regions and hence we get the desired result
$\lambda(\mathbf{H}) \in \cup_{l=1}^{L} [ \lambda_{min}(H_{ll}) - \mathcal{O}({max(d_l,d-d_l)}), \lambda_{max}(H_{ll}) + \mathcal{O}({max(d_l,d-d_l)})] $

\end{proof}
The general versions of the Gershgorin theorem to block matrices and is studied in~\cite{feingold1962block,salas1999gershgorin,tretter2008spectral}. We presented the proof
for the Hessian matrices, which are symmetric, by simply extending the usual Gershgorin circle theorem to block matrices. We bound the eigenvalues of the Hessian in terms of the min and max eigenvalues of the layerwise Hessian.

We consider the $C$ class classification problem. The training set $S = \{(\mathbf{x}^i,\mathbf{y}^i)\}_{i=1}^{N}$ is considered, where each $\mathbf{x}^i \in \mathbb{R}^D$ or $\mathbf{x}^i \in \mathbb{R}^{C\times H \times W}$and $\mathbf{y}^i \in \{0,1\}^C$ is drawn iid from the distribution $\mathcal{D}$. We then consider an $L$-layer neural network with ReLU non-linearity, where the network outputs the logits $\mathbf{z}^i$. The logits are obtained by a series of fully connected (FC)/convolutional (CONV) layers, followed by a non-linearity, represented concisely by Eq.~\ref{forward_affine} for an FC layer with parameters ($\{\mathbf{W}_l,\mathbf{b}_{l}\}$) and Eq.~\ref{forward_conv} for a CONV layer with parameters ($\{\mathbf{W}_l,\mathbf{b}_{l}\}$). We denote the collection of all the model parameters as $\theta \coloneqq \{\mathbf{W}_1, \mathbf{b}_1 ...\mathbf{W}_L, \mathbf{b}_L \}$ and $\theta \in \mathbb{R}^d$ where all the model parameters rolled into a single vector of dimension $d$.
Let $f_\theta(\mathbf{x^i})$ denote the final layer output of the model 
\begin{equation}
\mathbf{z}_l^i = \text{FC}(\mathbf{a}_{l-1}^i;\{\mathbf{W}_l,\mathbf{b}_{l}\})
\label{forward_affine}
\end{equation}
\begin{equation}
\mathbf{z}_l^i = \text{CONV}(\mathbf{a}_{l-1}^i;\{\mathbf{W}_l,\mathbf{b}_{l}\})
\label{forward_conv}
\end{equation}
\begin{equation}
\mathbf{a}_l^i = \sigma(\mathbf{z}_l^i) 
\end{equation}
where $\sigma(.)$ denotes non-linearity $\mathbf{a}_{0} = \mathbf{x}^i$, $\mathbf{z}^i = \mathbf{a}_L^i = f_\theta({\mathbf{x^i}})$. Finally, we use the cross-entropy loss 
\begin{equation}
\mathcal{L}(\mathbf{y}^i,\mathbf{z}^i) = \sum_{c=1}^{C}{-\mathbf{y}^i[c]log(\hat{\mathbf{y}}^i[c])}.
\label{softmax_eq}
\end{equation}
where $\hat{\mathbf{y}}$ is the softmax on logits $\mathbf{z}^i$ as 
\begin{equation}
\hat{\mathbf{y}}=  exp(\mathbf{z}^i)/\sum_{m=1}^{C}{exp(\mathbf{z}^i[m])})
\label{softmax_eqn}
\end{equation}
where $\mathbf{x}^i$,$\mathbf{y}^i$ denotes the $i^{th}$ input sample and label respectively. 
We use the notation $\mathcal{L}^i$ for $\mathcal{L}(\mathbf{y}^i,\mathbf{z}^i)$. The overall Loss computed on Batch size of B is denoted by 
\begin{equation}
\mathcal{L}={\frac{1}{B}}\sum_{i=1}^{B}{\mathcal{L}^i}
\label{sup:ce_loss}
\end{equation}

We have the following Lemma due to~\cite{wu2020dissecting}.
We denote $\theta \in \mathbf{R}^d$ as the collection of all the parameters, where $d$ denotes the total number of parameters.
\begin{lemma}
For the Network described in Eq.~\ref{forward_affine} to Eq.~\ref{sup:ce_loss}. The Hessian of loss $\mathcal{L}^i$ with respect to weights of FC layer  $\mathbf{W}_{l}$ denoted by $H_{\mathbf{W}_{l}}(\mathcal{L}^i)$ is given by $H_{\mathbf{W}_{l}}(\mathcal{L}^i) = \mathbf{M}_l(x^i,\theta) \otimes \mathbf{a}_{l-1}^i {\mathbf{a}^i}^{\top}_{l-1}$, where $\mathbf{M}_l(x^i)$ is a symmetric matrix.
\label{fc_lemma}
\end{lemma}
\begin{proof}
For the detailed derivation, please refer to Appendix A1 of~\cite{wu2020dissecting}. We present the proof for the sake of completeness.
We fix a layer $l$ for which we want to compute the Hessian, the inputs to layer $l$ is given by $\mathbf{a}^i_{l-1}$. The layer $l$ is parameterized by $\mathbf{W}_{l}$ and $\mathbf{b}_l$.
\begin{equation}
\mathbf{z^i}_l = \mathbf{W}_l\mathbf{a^i}_{l-1} + \mathbf{b}_l
\label{affine_eq}
\end{equation}

By using the chain rule for Hessian as\cite{skorski2019chain,wu2020dissecting} we get the following.
\begin{equation}
H_{\mathbf{w}_{l}}(\mathcal{L}^i) =  
{\frac{\partial {\mathbf{z}^i_l}} {\partial {\mathbf{w}_l}}}^\top   H_{\mathbf{z}^i} (\mathcal{L}^i) {\frac{\partial {\mathbf{z}^i_l}} {\partial {\mathbf{w}_l}}} + \sum_{n = 1}^{d_l} {\frac{\partial {l(\mathbf{z}^i,\mathbf{y}^i)}} {\partial {\mathbf{z}^i[n]}}} \nabla^2_{\mathbf{w}_l}{\mathbf{z}^i[n]}  
\end{equation}

Here $\mathbf{w}_{l} \coloneqq  vec(\mathbf{W}_{l})$, $\mathbf{z}^i[n]$ is the $n^{th}$ element of the vector $\mathbf{z}^i$.   $\nabla^2_{\mathbf{w}_l}{\mathbf{z}^i[n]}$ is Hessian of ${\mathbf{z}^i[n]}$ w.r.t ${\mathbf{w}_l}$. Also note that by convention $H_{\mathbf{w}_{l}}(\mathcal{L}^i)\coloneqq H_{\mathbf{W}_{l}}(\mathcal{L}^i)$ as we are only concerned with the Hessian of the loss w.r.t to the parameters of the layer $l$ not the structure in which these parameters are present.
 
From~\ref{affine_eq} we get  the following

\begin{equation}
{\frac{\partial {\mathbf{z}^i_l}} {\partial {\mathbf{w}_l}}} = \mathbf{I}_{d_l} \otimes {\mathbf{a}^i_{l-1}}^\top
\label{do2bydow}
\end{equation}

It is easy to see that $\nabla^2_{\mathbf{w}_l}{\mathbf{z}^i[n]} = 0$ and from~\ref{do2bydow} we have the following

\begin{equation}
H_{\mathbf{w}_{l}}(\mathcal{L}^i) = ({\mathbf{I}_{d_l} \otimes \mathbf{a}^i_{l-1}}) H_{\mathbf{z}^i} (\mathcal{L}^i) ({\mathbf{I}_{d_l} \otimes {\mathbf{a}^i_{l-1}}}^\top) 
\end{equation}

The above equation can be simplified as 
\begin{equation}
 H_{\mathbf{w}_{l}}(\mathcal{L}^i) = \mathbf{M}_l(x^i,\theta) \otimes
 \mathbf{a}^i_{l-1} {\mathbf{a}^i_{l-1}}^\top
\end{equation}
where $\mathbf{M}_l(x^i,\theta) \coloneqq H_{\mathbf{z}^i} (\mathcal{L}^i)$. It can be seen that $\mathbf{M}_l(x^i,\theta)$ is a symmetric matrix by definition. This concludes the proof
\end{proof}

Consider the CONV layer with input feature map of dimension $\mathbf{a}_{l-1}^i \in \mathbb{R}^{C_{l-1} \times H_{l-1} \times W_{l-1}}$, the output feature map $\mathbf{z}_l^i \in \mathbb{R}^{m \times H_l \times W_l}$ and convolutional kernel $\mathbf{W}_l \in \mathbb{R}^{m \times C_{l-1}\times K_1 \times K_2}$, we then have the following Lemma due to~\cite{wu2020dissecting}.

 We now state the two of our results that relate the layer-wise top eigenvalues to the activation norm of each layer. Consider a FC layer as in Eq.~\ref{forward_affine} with $\mathbf{a}_{l-1}^i \in \mathbb{R}^{d_{l-1}}$ and weights $\mathbf{W}_l \in \mathbb{R}^{d_{l}\times d_{l-1}}$. We then have the following result. 
\begin{theorem}
If ${\lVert \theta \rVert}_2 \leq \tilde{B}$ then the top eigenvalue of layer-wise Hessian for the loss $\mathcal{L}$ w.r.t to $ \textbf{W}_l$ denoted by $\lambda_{max}(\mathbf{H}_{\textbf{W}_{l}}{(\mathcal{L})}) $ for l = $2$ to $L$,  computed over the batch of samples for a $L$ layered fully connected neural network for multi-class classification is given by $\lambda_{max}({\mathbf{H}_{\textbf{W}_{l}}{(\mathcal{L})})   \leq  \alpha_l \sum_{i \in B} {\lVert \mathbf{a}_{l-1}^{i} \rVert}_2^2}$ where $\alpha_l>0$.
\label{hess_theorem1}
\end{theorem}

\begin{proof}
We use the results from the previous Lemma's and the fact that Hessian for the batch is the average of Hessian of all the individual samples.

\begin{equation}
\mathcal{L} = \frac{1}{B} \sum_{i \in B} \mathcal{L}^i
\label{batch_loss}
\end{equation}

\begin{equation}
\mathbf{H}_{\textbf{W}_l}({\mathcal{L}}) = \frac{1}{B} \sum_{i \in B} \mathbf{H}_{\textbf{W}_l}({\mathcal{L}^i}) 
\label{hessian_W_l_batch}
\end{equation}

By repeated application of Eq.~\ref{pl5} to Eq.~\ref{hessian_W_l_batch}, we have the following.

\begin{equation}
\lambda_{max}(\mathbf{H}_{\textbf{W}_l}({\mathcal{L}})) \leq
{1 \over B} \sum_{i \in B} \lambda_{max}(\mathbf{H}_{\textbf{W}_l}({\mathcal{L}^i}))
\label{sup:simphess1}
\end{equation}

From the Lemma~\ref{fc_lemma} we have the following
\begin{equation}
H_{\mathbf{W}_{l}}(\mathcal{L}^i) = \mathbf{M}_l(x^i,\theta) \otimes \mathbf{a}_{l-1}^i {\mathbf{a}^i}^{\top}_{l-1}
\label{sup:lemma1_eq}
\end{equation}

By using Eq.~\ref{pl4} in the above Eq.~\ref{sup:lemma1_eq} we get the following. 

\begin{equation}
\lambda_{max}(\mathbf{H}_{\textbf{W}_l}({\mathcal{L}^i})) = \lambda_{max}(\mathbf{M}_l(x^i,\theta)) \lambda_{max}(\mathbf{a}^i_{l-1} {\mathbf{a}^i}^{\top}_{l-1})
\label{sup:sample_eig}
\end{equation}

We now show that the $\lambda_{max}(\mathbf{M}_l(\mathbf{x}^i,\theta))$
exists and its finite in the following arguments.

From~\cite{ams_article}, we know that the eigenvalues are the continuous functions of the coefficients of characteristic polynomials, and so is the top eigenvalue. 

Since every entry in the matrix $\mathbf{M}_l(\mathbf{x}^i,\theta)$ is a continuous function of $\theta$. The coefficients of the characteristic polynomials are also continuous functions of $\theta$. Since the continuity is preserved under the composition of continuous functions, i.e., if $f$ is continuous and $g$ is continuous, then the composition $fog$ is continuous.

If top eigenvalue $\lambda_{max}$ is a continuous function of the coefficients of characteristic polynomial and the coefficients are again a continuous function of the variable $\theta$. We conclude that $\lambda_{max}$ is a continuous function $\theta$. \\
The set $\{ \theta \colon {\lVert \theta \rVert}_2 \leq \tilde{B}\}$ where $\theta \in \mathbb{R}^d$ is compact. Continuous function map compact sets to compact sets. Thus the function $\lambda_{max}$ attains its supremum and its finite.
\begin{equation}
\lambda_{max}(\mathbf{M}_l(\mathbf{x}^i,\theta)) \leq \underset{\theta} \sup (\lambda_{max}(\mathbf{M}_l(\mathbf{x}^i,\theta))) = \alpha_l^{i}
\label{sup:eig_bound}
\end{equation}

We note that $\alpha_l^{i} > 0$, suppose if its negative, by~\ref{sup:sample_eig} we see that top eigenvalue of layerwise Hessian is negative; this implies the loss function is concave, which contradicts the fact that neural-networks are non-convex and non-concave functions.

By using the bound in Eq.~\ref{sup:eig_bound} in Eq.~\ref{sup:sample_eig}, we can bound the Eq.~\ref{sup:simphess1} as below.

\begin{equation}
\lambda_{max}(\mathbf{H}_{\textbf{W}_l}({\mathcal{L}})) \leq
{1 \over B} \sum_{i \in B}  \alpha_l^i {\lVert \mathbf{a}^i_{l-1} \rVert}_{F}^2 
\label{final_hessian3}
\end{equation}
we have the following, where  $\alpha_l$ is the maximum over all the training samples $\alpha_l^{i}$.
\begin{equation}
\alpha_l^{i} \leq \underset{i} \max (\alpha_l^{i}) = \alpha_l
\label{sup:alpha_bound}
\end{equation}

Using Eq.~\ref{sup:alpha_bound} in the Eq.~\ref{final_hessian3} we get the following

\begin{equation}
\lambda_{max}(\mathbf{H}_{\textbf{W}_l}({\mathcal{L}})) \leq
{1 \over B} \alpha_l \sum_{i \in B} {\lVert \mathbf{a}^i_{l-1} \rVert}_{2}^2 
\label{final_eigen_value}
\end{equation}
This completes the proof.
\end{proof}

\begin{lemma}
For the Network described in Eq.~\ref{forward_affine} to Eq.~\ref{sup:ce_loss}. The Hessian of loss $\mathcal{L}^i$ with respect to weights of CONV layer  $\mathbf{W}_{l}$ denoted by $H_{\mathbf{W}_{l}}(\mathcal{L}^i)$ is approximated by $H_{\mathbf{W}_{l}}(\mathcal{L}^i) \approx \tilde{\mathbf{M}}_l(x^i) \otimes \mathbf{a}_{l-1}^i {\mathbf{a}^i}^{\top}_{l-1}$. 
\label{conv_lemma}
\end{lemma}
\begin{proof}
For a detailed discussion, please refer to Appendix A.2 of~\cite{wu2020dissecting}.
\end{proof}

\begin{theorem} 
If ${\lVert \theta \rVert}_2 \leq \tilde{B}$, the top eigenvalue of layer-wise Hessians for the loss (w.r.t to $ \textbf{W}_l$ for l = $2$ to $L$ ) computed over the Batch of samples for a $L$ layered convolutional neural network for multi-class classification is given by $\lambda_{max}({\mathbf{H}_{\textbf{W}_{l}}({\mathcal{L}}))  \leq  \alpha_l \sum_{i \in B} {\lVert \textbf{a}_{l-1}^{i} \rVert}_F^2}$ where where $\alpha_l>0$.
\label{hess_theorem2}
\end{theorem}

\begin{proof}
In the proof technique, we follow the exact similar steps as the above theorem with some minor changes. The major change here is we now use the convolutional layers. 
\begin{equation}
\mathcal{L} = {1 \over B} \sum_{i \in B} \mathcal{L}^i
\label{th2:batch_loss}
\end{equation}

\begin{equation}
\mathbf{H}_{\textbf{W}_l}({\mathcal{L}}) = {1 \over B} \sum_{i \in B} \mathbf{H}_{\textbf{W}_l}({\mathcal{L}^i}) 
\label{th2:hessian_W_l_batch}
\end{equation}

By repeated application of Eq.~\ref{pl5} to the Eq.~\ref{th2:hessian_W_l_batch} we have the following.

\begin{equation}
\lambda_{max}(\mathbf{H}_{\textbf{W}_l}({\mathcal{L}})) \leq
{1 \over B} \sum_{i \in B} \lambda_{max}(\mathbf{H}_{\textbf{W}_l}({\mathcal{L}^i}))
\label{th2:simphess1}
\end{equation}

From the Lemma~\ref{conv_lemma} we have the following
\begin{equation}
H_{\mathbf{W}_{l}}(\mathcal{L}^i) = \tilde{\mathbf{M}}_l(x^i) \otimes \mathbf{a}_{l-1}^i {\mathbf{a}^i}^{\top}_{l-1}
\label{th2:lemma1_eq}
\end{equation}

By using Eq.~\ref{pl4}  in the above Eq.~\ref{th2:lemma1_eq} we get the following. 

\begin{equation}
\lambda_{max}(\mathbf{H}_{\textbf{W}_l}({\mathcal{L}^i})) = \lambda_{max}(\tilde{\mathbf{M}}_l(x^i)) \lambda_{max}(\mathbf{a}^i_{l-1} {\mathbf{a}^i}^{\top}_{l-1})
\label{th2:sample_eig}
\end{equation}

We cannot use Eq.~\ref{out_eig_eq} directly to find the eigenvalue of $\mathbf{a}^i_{l-1} {\mathbf{a}^i}^{\top}_{l-1}$ as $\mathbf{a}^i_{l-1}$ is matrix not a vector, because we are dealing with convolutional layers.
Since $\mathbf{a}^i_{l-1} {\mathbf{a}^i}^{\top}_{l-1}$ is a positive semi-deifinite matrix we have the following inequality

\begin{equation}
 \lambda_{max}(\mathbf{a}^i_{l-1} {\mathbf{a}^i}^{\top}_{l-1}) \leq \text{Trace}(\mathbf{a}^i_{l-1} {\mathbf{a}^i}^{\top}_{l-1})
 \label{trace_eig_ineq}
\end{equation}

By using the identity $\text{Trace}(\mathbf{a}^i_{l-1} {\mathbf{a}^i}^{\top}_{l-1}) = {\lVert \mathbf{a}^i_{l-1} \rVert}_{F}^2 $ in the Eq.~\ref{th2:sample_eig} we get the following 

\begin{equation}
\lambda_{max}(\mathbf{H}_{\textbf{W}_l}({\mathcal{L}^i})) \leq \lambda_{max}(\tilde{\mathbf{M}}_l(x^i)) {\lVert \mathbf{a}^i_{l-1} \rVert}_{F}^2
\label{th2:sample_eig_ineq}
\end{equation}

We can use similar reasoning as in Theorem~\ref{hess_theorem1} to bound the value of the $\lambda_{max}(\tilde{\mathbf{M}}_l(x^i))$.

By substituting the above inequality~\ref{th2:sample_eig_ineq} in~\ref{th2:simphess1} we get the following.

\begin{equation}
\lambda_{max}(\mathbf{H}_{\textbf{W}_l}({\mathcal{L}})) \leq
{1 \over B} \sum_{i \in B}  \alpha_l^i {\lVert \mathbf{a}^i_{l-1} \rVert}_{F}^2 
\label{th2:final_hessian2}
\end{equation}
where we have used the fact $ \lambda_{max}(\tilde{\mathbf{M}}_l(x^i)) \leq \alpha_l^i$.

If we denote $\alpha_l$ as the maximum overall $\alpha_l^i$ over the batch.
We then get the following

\begin{equation}
\lambda_{max}(\mathbf{H}_{\textbf{W}_l}({\mathcal{L}})) \leq
{1 \over B} \alpha_l \sum_{i \in B} {\lVert \mathbf{a}^i_{l-1} \rVert}_{F}^2 
\label{th2:final_eigen_value}
\end{equation}
This completes the proof.
\end{proof}


\section{Model Architectures}
In Table~\ref{model_table}, the model architecture is shown. We use PyTorch style representation. For example convolutional (CONV) layer($3$,$64$,$5$) means $3$ input channels, $64$ output channels and the kernel size is $5$. Maxpool($2$,$2$) represents the kernel size of $2$ and a stride of $2$. Fully Connected (FC)(384,200) represents an input dimension of $384$ and an output dimension of $200$. The architecture for CIFAR-100 is exactly the same as used in~\cite{acar2021federated}. 
\begin{table}[htp]
\caption{Models used for Tiny-ImageNet and CIFAR-100 datasets.}
\centering
\begin{tabular}{c|c}
\hline
\multicolumn{1}{l|}{\multirow{6}{*}{\textbf{CIFAR-100 Model}}} & \textbf{Tiny-ImageNet Model} \\ \cline{2-2} 
\multicolumn{1}{l|}{}                                              & ConvLayer(3,64,3)            \\ \cline{2-2} 
\multicolumn{1}{l|}{}                                              & GroupNorm(4,64)               \\ \cline{2-2} 
\multicolumn{1}{l|}{}                                              & Relu                          \\ \cline{2-2} 
\multicolumn{1}{l|}{}                                              & MaxPool(2,2)                  \\ \cline{2-2} 
\multicolumn{1}{l|}{}                                              & ConvLayer(64,64,3)           \\ \hline
\multicolumn{1}{l|}{}                                              & GroupNorm(4,64)               \\ \hline
ConvLayer(3,64,5)                                                  & Relu                          \\ \hline
Relu                                                                & MaxPool(2,2)                  \\ \hline
MaxPool(2,2)                                                        & ConvLayer(64,64,3)           \\ \hline
ConvLayer(64,64,5)                                                 & GroupNorm(4,64)               \\ \hline
Relu                                                                & Relu                          \\ \hline
MaxPool(2,2)                                                        & MaxPool(2,2)                  \\ \hline
Flatten                                                             & Flatten                       \\ \hline
FullyConnected(1600,384)                                           & FullyConnected(4096,512)     \\ \hline
Relu                                                                & Relu                          \\ \hline
FullyConnected(384,192)                                            & FullyConnected(512,384)      \\ \hline
Relu                                                                & Relu                          \\ \hline
FullyConnected(192,100)                                                     & FullyConnected(384,200)               \\ \hline
\end{tabular}
\label{model_table}
\end{table}
\vspace{-0.1in}
\section{Sensitivity to hyper-parameter $\zeta$}
\begin{figure}[htp]
    \centering
    \includegraphics[scale = 0.6]{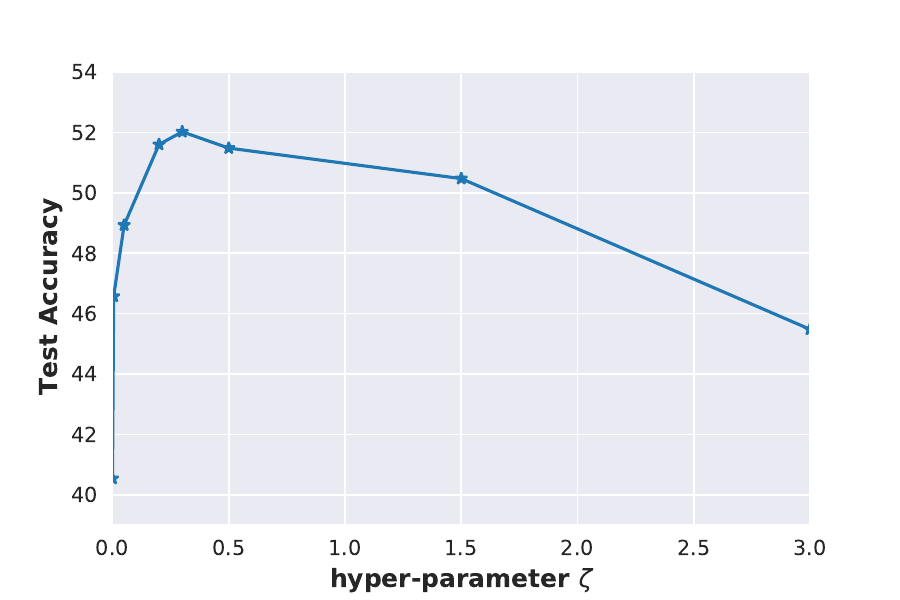}
    \caption{Sensitivity of Accuracy to the hyper-parameter $\zeta$. It can be seen that accuracy is stable over $\zeta \in \{ 0.1, 3.0\}$ }
    \label{fig:sens_hyper_parameter}
\end{figure}
In figure~\ref{fig:sens_hyper_parameter}, we perform sensitivity analysis on the hyper-parameter $\zeta$ i.e. how model accuracy varies over different values of $\zeta$. We consider the algorithm FedAvg+ MAN on CIFAR-100 dataset with Dirichlet-based non-iid data partition ($\delta = 0.3$). We observe that accuracy is stable over $\zeta \in \{0.1, 1.5\}$.

\section{Additional Results}
\subsection{Results for FedSAM/ASAM and FedSpeed with and without MAN}
\begin{figure*}[htp]
  \centering
  \subfloat[$\delta = 0.3$]{
  \includegraphics[scale=0.39]{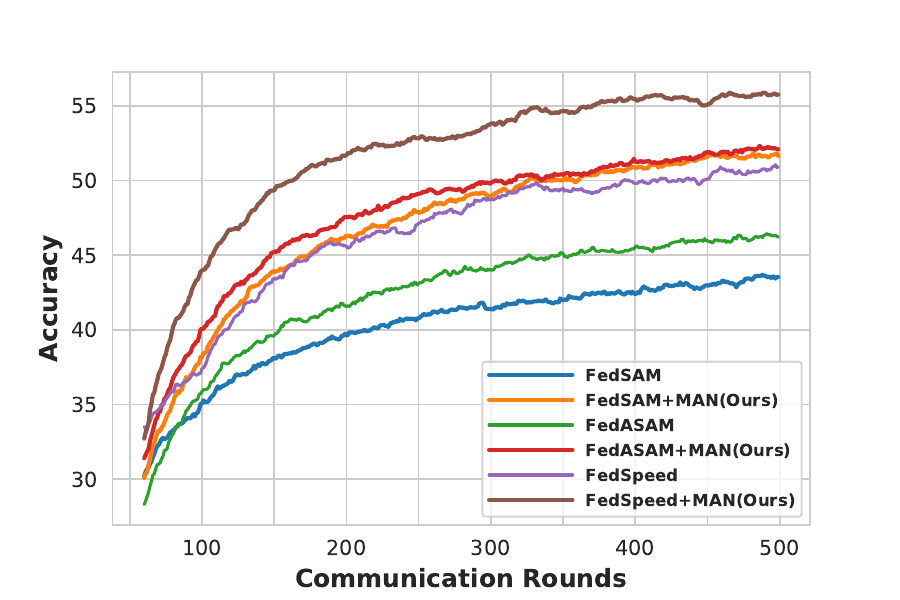}}\hspace{-2.em}
  \subfloat[$\delta = 0.6$]{
  \includegraphics[scale=0.39]{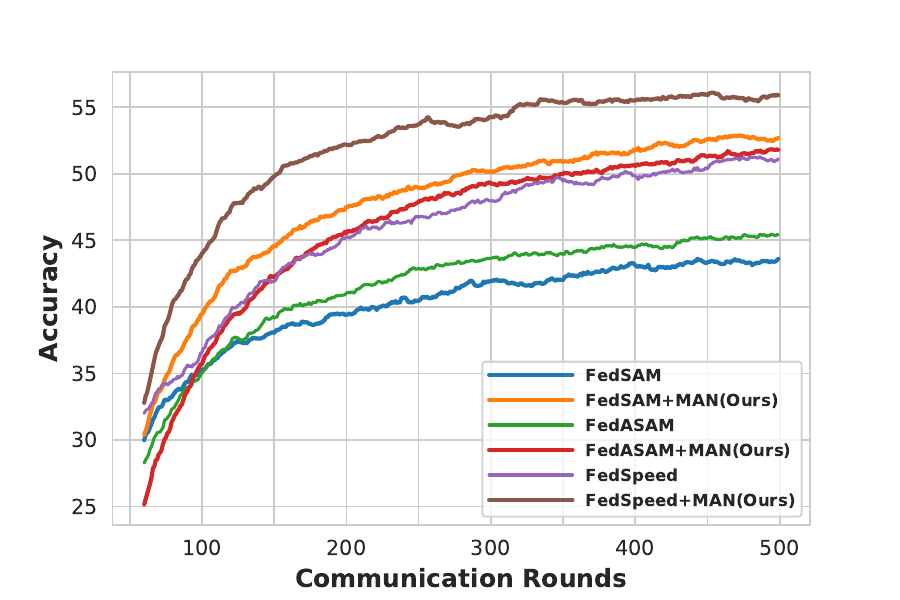}}\hspace{-2.em}
  \subfloat[iid]{
  \includegraphics[scale=0.39]{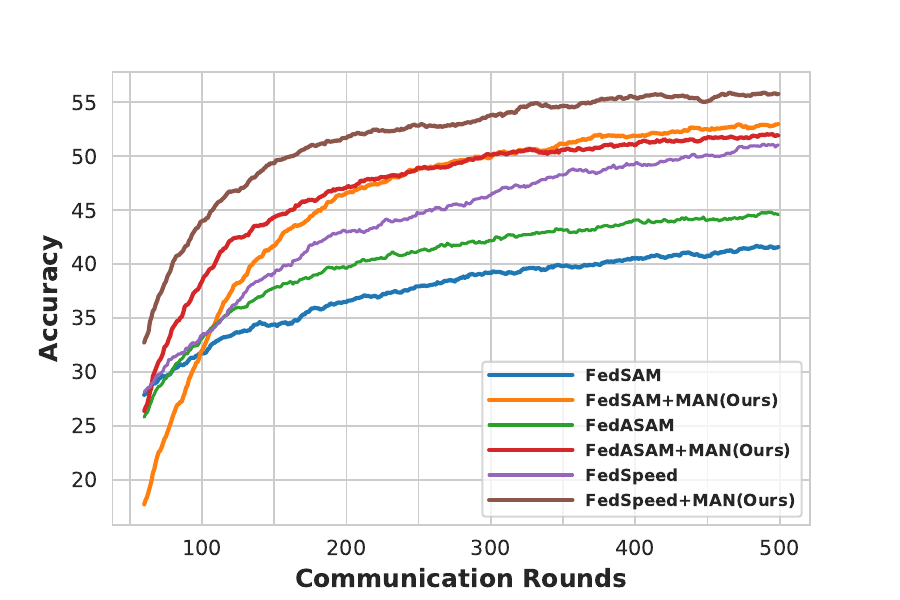}
  }
  \caption{Convergence Comparison for CIFAR-100: We compare performance of the algorithms FedAvg, FedDyn, FedDC and the proposed FedAvg+MAN, FedDyn+MAN and FedDC+MAN for 500 communication rounds. It can be clearly seen that  proposed approach significantly improves the existing algorithms across the communication rounds. }
  \label{set2_cifar100_perf}
\end{figure*}
\begin{figure*}[htp]
  \centering
  \subfloat[$\delta = 0.3$]{
  \includegraphics[scale=0.39]{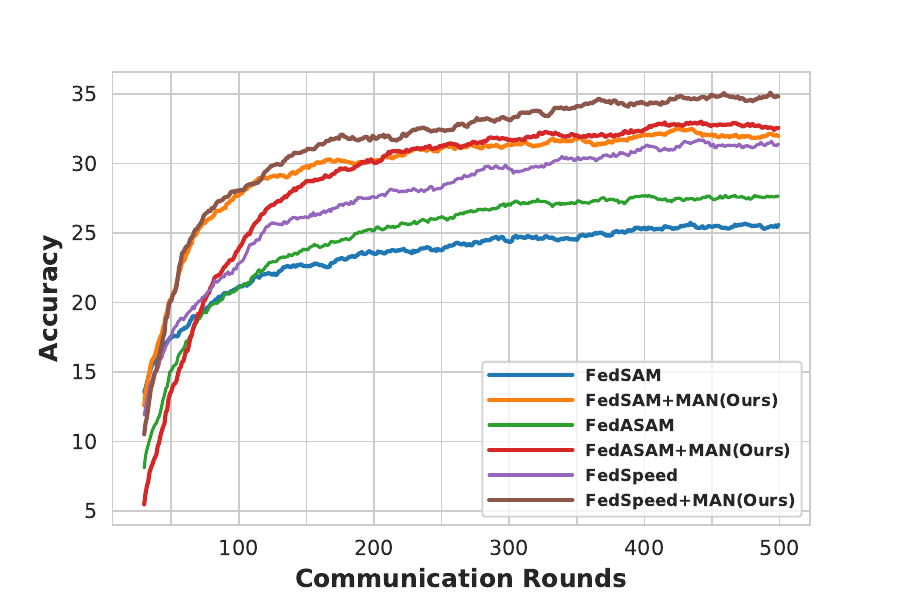}} \hspace{-2.em}
  \subfloat[$\delta = 0.6$]{
  \includegraphics[scale=0.39]{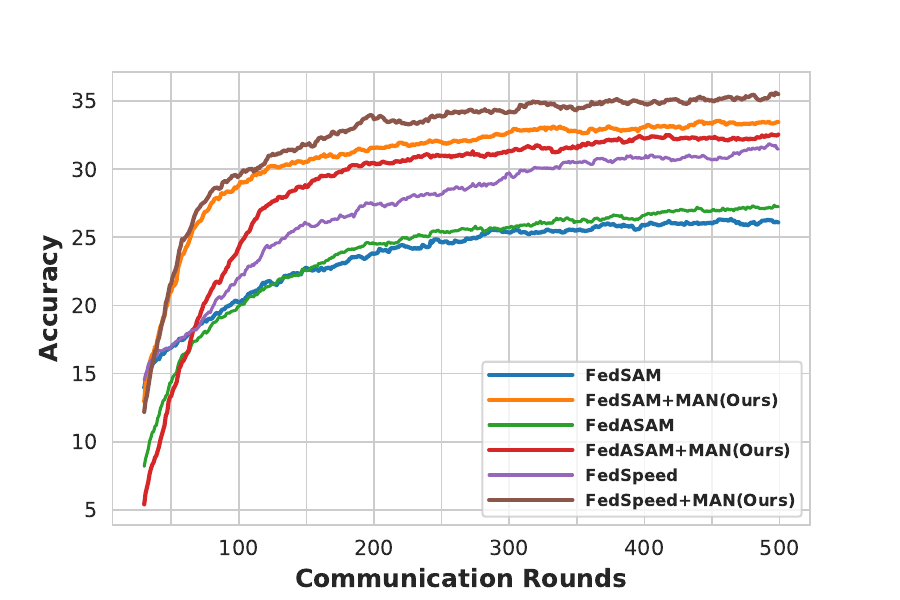}}
  \hspace{-2.em}
  \subfloat[iid]{
  \includegraphics[scale=0.39]{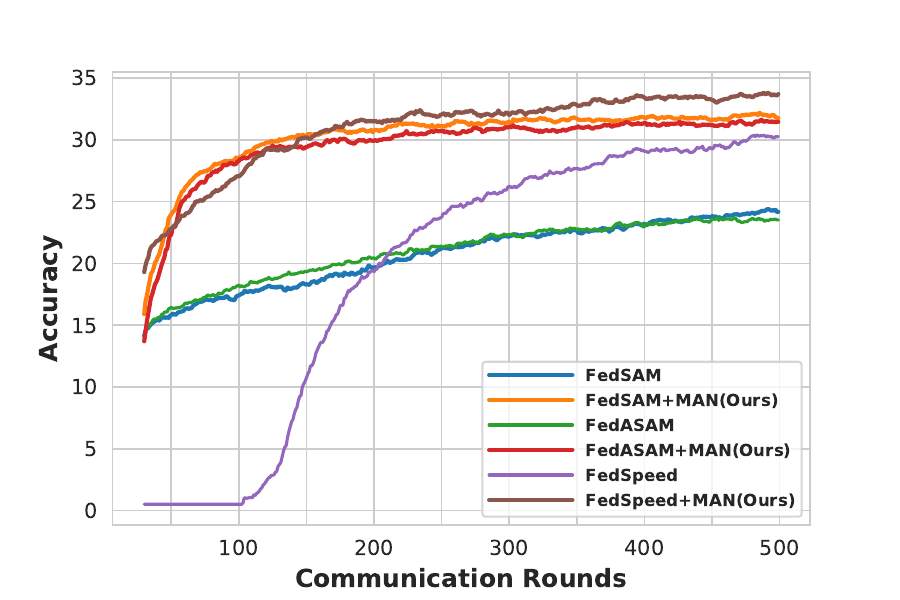}
  }
  \caption{Convergence Comparison for Tiny-ImageNet: We compare the performance of the algorithms FedAvg, FedDyn, FedDC and the proposed FedAvg+MAN, FedDyn+MAN, and FedDC+MAN for 500 communication rounds. It can be clearly seen that the proposed approach significantly improves the existing algorithms.}
  \label{set2_TinyImageNet_perf}
\end{figure*}

We have developed all our experiments based on the open source code provided by~\cite{acar2021federated}. In the figure~\ref{set2_cifar100_perf}, we provide the plots for communication rounds vs Accuracy on CIFAR-100 dataset for FedSAM, FedASAM and FedSpeed with and without using our MAN regularizer. It can be seen that MAN regularizer consistently improves the performance of all the algorithms. Similar results for Tiny-ImageNet are provided in the figure~\ref{set2_TinyImageNet_perf}. We can observe consistent improvement in the performance of the algorithms when MAN regularizer is added. 
\subsection{CIFAR-10 Results}

In this section we provide the results for CIFAR-10 dataset. In the table~\ref{cifar10_results} we report the performance of all the algorithms (FedAvg,FedDyn,FedDC, FedSAM,FedASAM and FedSpeed) with and without using the MAN. It can be clearly seen that the performance of all the algorithms can be improved when our MAN regularizer is used with the algorithms. 
\begin{table}[htp]
\caption{Comparison of various methods with and without MAN regularizer with different degrees of heterogeneity on CIFAR-10 dataset. MAN clearly improves the performance consistently across all the methods. All the experiments are repeated for three different initializations and their mean and standard deviations are reported.}
\scalebox{0.8}{
\begin{tabular}{l|lll}
\hline
\multirow{2}{*}{Method} & \multicolumn{3}{c}{CIFAR10}                                                         \\ \cline{2-4} 
                        & \multicolumn{1}{c|}{$\delta = 0.6$}        & \multicolumn{1}{c|}{$\delta = 0.3$}        & iid          \\ \toprule
FedAvg                  
& \multicolumn{1}{l}{79.34{$_{\pm{0.19}}$}}  
& \multicolumn{1}{l}{80.19 $_{\pm{0.46}}$} 
& 81.44$_{\pm{0.43}}$  
\\ 
FedAvg+MAN              
& \multicolumn{1}{l}{\textbf{82.53}$_{\pm{0.25}}$}  
& \multicolumn{1}{l}{\textbf{83.35}$_{\pm{0.09}}$}  
& \textbf{84.18} $_{\pm{0.15}}$  
\\ \midrule
FedSAM                  
& \multicolumn{1}{l}{80.42 $_{\pm{0.47}}$}  
& \multicolumn{1}{l}{81.40 $_{\pm{0.17}}$}  
& 82.54 $_{\pm{0.14}}$  
\\ 
FedSAM+MAN              
& \multicolumn{1}{l}{\textbf{81.56} $_{\pm{0.1}}$}   
& \multicolumn{1}{l}{\textbf{82.54} $_{\pm{0.16}}$}  
& \textbf{84.19} $_{\pm{0.26}}$  
\\ \midrule
FedASAM                 
& \multicolumn{1}{l}{79.90 $_{\pm{0.43}}$} 
& \multicolumn{1}{l}{80.83 $_{\pm{0.08}}$} 
& 82.27 $_{\pm{0.45}}$
\\
FedASAM+MAN             
& \multicolumn{1}{l}{\textbf{80.81} $_{\pm{0.06}}$}  
& \multicolumn{1}{l}{\textbf{81.85} $_{\pm{0.08}}$}  
& \textbf{84.22} $_{\pm{0.1}}$  

\\ \midrule
FedDyn                  
& \multicolumn{1}{l}{82.00 $_{\pm{0.22}}$} 
& \multicolumn{1}{l}{82.53 $_{\pm{0.06}}$} 
& 84.16 $_{\pm{0.41}}$ 
\\ 
FedDyn+MAN              
& \multicolumn{1}{l}{\textbf{83.84}$_{\pm{0.38}}$}  
& \multicolumn{1}{l}{\textbf{84.63}$_{\pm{0.17}}$}  
& \textbf{84.80}$_{\pm{0.25}}$  
\\ \midrule
FedDC                   
& \multicolumn{1}{l}{83.10 $_{\pm{0.37}}$} 
& \multicolumn{1}{l}{83.64 $_{\pm{0.13}}$} 
& 84.8 $_{\pm{0.21}}$  
\\
FedDC+MAN               
& \multicolumn{1}{l}{\textbf{83.27}$_{\pm{0.18}}$}  
& \multicolumn{1}{l}{\textbf{83.59} $_{\pm{0.15}}$}  
& \textbf{85.08} $_{\pm{0.1}}$   
\\ \midrule
FedSpeed                
& \multicolumn{1}{l}{84.06 $_{\pm{0.11}}$} 
& \multicolumn{1}{l}{84.24 $_{\pm{0.14}}$} 
& 85.14 $_{\pm{0.3}}$   
\\ 
FedSpeed+MAN            
& \multicolumn{1}{l}{\textbf{84.37} $_{\pm{0.25}}$}  
& \multicolumn{1}{l}{\textbf{84.82} $_{\pm{0.16}}$}  
& \textbf{85.95} $_{\pm{0.24}}$  
\\ \hline
\end{tabular}
}

\label{cifar10_results}
\end{table}

\subsection{Empirical Hessian Analysis}
\begin{table}[htp]
\centering
\caption{Comparison of top eigenvalues and trace of the algorithms with and without MAN regularizer, lower values are better. We can observe that by
augmenting MAN regularization, i.e., FedASAM+MAN, FedSpeed+MAN, we obtain lower trace and lower top eigenvalues, which is indicative of flat minimum, and hence, it attains better accuracy.}
\scalebox{0.9}{
\begin{tabular}{l|cccc}
\toprule
                         & \multicolumn{4}{c}{CIFAR-100}                                                                                                                       \\ \cline{2-5} 
                         & \multicolumn{2}{c|}{$\delta = 0.3$}                                   & \multicolumn{2}{c}{$\delta = 0.6$}                                                \\ \cline{2-5} 
\multirow{-3}{*}{Method} & \multicolumn{1}{l}{\begin{tabular}[c]{@{}c@{}}Top \\ eigenvalue\end{tabular}} & \multicolumn{1}{l|}{Trace}   & \multicolumn{1}{l}{\begin{tabular}[c]{@{}c@{}}Top \\ eigenvalue\end{tabular}}            & \multicolumn{1}{l}{Trace}     \\ \midrule
FedASAM                   
& \multicolumn{1}{c}{51.49}          
& \multicolumn{1}{c|}{8744} 
& \multicolumn{1}{c}{53.39}                     
& 9056                        
\\ 
FedASAM+MAN               
& \multicolumn{1}{c}{\textbf{43.80}}        
& \multicolumn{1}{c|}{\textbf{4397}}
& \multicolumn{1}{c}{\textbf{42.00}} 
& {\textbf{4747}} 
\\ \hline
FedSpeed                  
& \multicolumn{1}{c}{47.31}          
& \multicolumn{1}{c|}{6463} 
& \multicolumn{1}{c}{49.75}                     
& 6519                        \\ 
FedSpeed+MAN               
& \multicolumn{1}{c}{\textbf{40.41}}          
& \multicolumn{1}{c|}{\textbf{3806}} 
& \multicolumn{1}{c}{\textbf{37.64}}                     
& \textbf{3674}   
\\ \bottomrule
\end{tabular}
}
\label{sam_analysis_sup}
\end{table}

In Table~\ref{sam_analysis_sup} we present the top eigenvalue and the trace of the Hessian of the loss of global model for FedASAM, FedSpeed and their improved versions using MAN i.e, FedASAM+MAN and FedSpeed+MAN.

\section{Non-iid Data generation}
We now briefly describe how the data is generated using the Dirichlet distribution. This distribution is parameterized by parameter $\delta$. For every client, we sample a vector for the Dirichlet distribution. This vector is of the length of a total number of classes and represents the label distribution of the clients. The lower value of $\delta$ implies high heterogeneity, i.e.; label distribution is non-uniform; only a few labels dominate the samples on the client's data. 

\section{Hyper-parameter settings}
All the algorithms use a learning rate of $0.1$, batch size of $50$, client participation of $10\%$, and a gradient clipping threshold of $10$. We use $5$ local epochs for client training learning rate decay of $0.998$ for every round was used. 

For FedAvg+MAN, we use $\zeta = 0.6$ by default. For FedDyn, we use $\alpha = 0.01$. For FedDC also uses $\alpha = 0.01$. For FedSpeed, we use $\rho = 0.1$ for non-iid settings and $\rho = 0.01$ for iid setting, $\beta = 1.0$ ,$\gamma = 1.0$ and no gradient cutoff threshold to 0.05. For FedSAM we use $\rho = 0.05$. Only for Tiny-ImageNet with iid partition, we set $\rho= 0.03$. When we use MAN, i.e, FedSAM+MAN, we set $\rho=0.01$ for Tiny-ImageNet with iid partition, and weight decay of $1e-3$ is used.
For FedASAM we use $\rho = 0.5$ and $\eta = 0.2$. Only for iid partition we set the values to $\rho = 0.1$ and $\eta = 0.2$.

\section{Algorithm details of FedAvg+MAN, FedDyn+MAN and FedDC+MAN}
We present the algorithm details for implementing FedAvg+MAN, FedDyn+MAN, and FedDC+MAN 
respectively. Each client minimizes the activation norm as a regularizer in the algorithm and the cross-entropy loss, as shown in the below Eq.~\ref{sup:eq1}.
\begin{equation}
f_{k}(\mathbf{w}) \triangleq L_{k}(\mathbf{w}) + \zeta L_{k}^{act}(\mathbf{w})
\label{sup:eq1}
\end{equation}
$ L_{k}(\mathbf{w})$ denotes the task-specific loss in our case, it is (cross-entropy loss) and $ L_{k}^{act}(\mathbf{w})$ is the activation norm loss that is used to attain flatness and is described in detail in the Sec.3.2.3 of the main paper. The hyper-parameter $\zeta$ trades off between the flatness and the cross-entropy loss. In this way, it is straightforward to integrate the proposed regularizer 'MAN' into the existing FL algorithms. 
The complete details of individual algorithms can be found in FedAvg~\cite{mcmahan2017communication}, FedDyn~\cite{acar2021federated} and FedDC~\cite{Gao_2022_CVPR}. 
We can similarly extend the MAN to FedSAM/ASAM and FedSpeed as well. We simply add activation norms to the client loss function as in Eq.~\ref{sup:eq1}.

\end{document}